\documentclass{article}

\usepackage{amssymb,amsthm}
\usepackage{graphicx,calc}
\usepackage{natbib}
\usepackage[top=3cm,inner=3cm,outer=3cm,bottom=4cm]{geometry}

\usepackage{hyperref}
\usepackage{url}
\usepackage{amsmath}

\usepackage{color}
\usepackage{enumitem}
\usepackage{upgreek}
\usepackage{epstopdf}

\usepackage{supertabular}

\usepackage[normalem]{ ulem }
\usepackage{soul}

\let\originalleft\left
\let\originalright\right
\renewcommand{\left}{\mathopen{}\mathclose\bgroup\originalleft}
\renewcommand{\right}{\aftergroup\egroup\originalright}

\newtheorem{theorem}{Theorem}[section]
\newtheorem{example}[theorem]{Example}
\newtheorem{lemma}[theorem]{Lemma}
\newtheorem{definition}[theorem]{Definition}
\newtheorem{remark}[theorem]{Remark}

\newcommand{\paren}[1]{\left(#1\right)}
\newcommand{\brac}[1]{\left[#1\right]}
\newcommand{\inner}[1]{\left\langle#1\right\rangle}
\newcommand{\norm}[1]{\left\|#1\right\|}
\newcommand{\set}[1]{\left\{#1\right\}}
\newcommand{\abs}[1]{\left\lvert #1 \right\rvert}

\newcommand{\wt}[1]{\widetilde{#1}}

\def \Exp {\mathbb{E}}

\def \Med {\mathrm{Med}}
\def \Moy {\mathrm{Mean}}

\newcommand{\eps}{\varepsilon}

\def \trace{\operatorname{Tr}}

\def \RR {\mathbb{R}}
\def \CC {\mathbb{C}}

\def \Simplex {\mathbb{S}}

\def \SketchingOperator {\mathrm{\Phi}}
\def \SketchingOperatorProb {\mathcal{A}}

\def \nMeasures {m}
\def\nSamples{n}
\def\sampleDim{d}

\def\sample{x}
\def\Sample{X}

\def\dataset{\mathbf{X}}

\def \Prob {\pi}

\def \mProb {{\tau}}
\def \empProb{\hat{\Prob}_{\nSamples}}
\def \Param {\theta}
\def \HH {\mu}
\def \estProb{\widetilde{\Prob}}
\def \FSMSpace {\mathfrak{M}}

\def \SampleSpace{\mathcal{Z}}

\def \FClass {\mathcal{F}}
\def \GClass {\mathcal{G}}
\def \Model {\mathfrak{S}}
\def \ModelML {\Model^{\mathtt{ML}}}
\def \ModelCT {\Model^{\mathtt{CT}}}
\newcommand{\sep}{\eps}
\def\secant{\mathcal{S}}
\def \distIOP {d^{\circ}}
\def \distIOPgen {\mathcal{D}^*}

\def \distIOPexgen {d}
\newcommand\normfclass[2]{\norm{#1}_{\mathcal{#2}}}
\newcommand\normlossf[2]{\norm{#1}_{\LossClass(\HypClass_{#2})}}
\newcommand\normloss[2]{\norm{#1}_{\LossClass}}
\newcommand\dnormloss[2]{\norm{#1}_{\Delta\LossClass}}
\newcommand\normdloss[2]{\norm{#1}_{\DLossClass(\HypClass_{#2})}}
\newcommand\normmah[2]{\norm{#1}_{#2}}
\newcommand\normkern[1]{\norm{#1}_{\kernel}}
\newcommand{\normTV}[1]{\norm{#1}_{\textnormal{TV}}}
\newcommand\KLdiv[2] {\textnormal{KL}(#1\|#2)}

\newcommand{\distp}{d}
\newcommand{\dX}{d_{\mathcal{X}}}
\newcommand{\dY}{d_{\mathcal{Y}}}

\def \loss {\ell}

\def \LossClass {\mathcal{L}}
\def \DLossClass {\Delta\mathcal{L}}
\def \hyp {h}
\def \HypClass{\mathcal{H}}

\def \LipClass {\mathrm{Lip}}

\def \Risk {\mathcal{R}}

\def \proxyRisk {R}

\def \Decoder {\Delta}
\def \Sketch {\mathtt{Sketch}}
\def \Learn {\mathtt{Learn}}

\def \ConcFn {c_{\kernel}}
\def \covnumsymbol {\mathrm{N}}
\newcommand{\covnum}[3]{\covnumsymbol\left(#1,#2,#3\right)}

\def \coveps {\delta}
\def \probLevel {\zeta}

\newcommand{\drisk}{\Delta \Risk}

\newcommand{\divp}{D}
\newcommand{\divg}{d}

\newcommand{\projspace}{\mathfrak{P}}
\def \kernel {\kappa}

\def \freqdist {\Lambda}
\def \freq {\omega}

\def \freqSpace{\Omega}
\def \feat {\phi}
\def \rfeat {{\feat_{\omega}}}

\def \rfeatj {{\feat_{\omega_j}}}

\def \vx {\mathbf{x}}
\def \vy {\mathbf{y}}

\def \vu {\mathbf{u}}

\def \mSigma {{\boldsymbol\Sigma}}

\def \mA {\mathbf{A}}

\def \mX {\mathbf{X}}

\def \mL {\mathbf{L}}

\def \mI {\mathbf{I}}

\def \Ball {\mathcal{B}}

\def \covar {\mSigma}

\def \Cov {\mathbf{\Sigma}}
\def \eigv {\lambda}

\def \argmin {\mathrm{argmin}}

\definecolor{darkpurple}{rgb}{0.3,0,0.3}

\def \PCAdim {k}
\newcommand{\kPCArisk}{\Risk_{k-\mathtt{PCA}}}
\newcommand{\projP}{\mathbf{P}}
\newcommand{\optproj}{\projP^{*[\PCAdim]}_{\Prob}}
\newcommand{\optprojl}{\projP^{*[\ell]}_{\Prob}} 
 \ifdefined \highlightedchanges
 \newcommand\rev[1]{\textcolor{blue}{#1}}
 \else
 \newcommand\rev[1]{{#1}}
 \fi

\def\citeparttwo{\citep{gribonval:CSL2}}

\begin{document}

\title{Compressive Statistical Learning\protect\\ with Random Feature Moments}
\author{
R{\'e}mi Gribonval\thanks{Univ Lyon, Inria, CNRS, ENS de Lyon, UCB Lyon 1, LIP UMR 5668, F-69342, Lyon, France\protect\\
  This work was initiated while R. Gribonval, N. Keriven and Y. Traonmilin were with Univ Rennes, Inria, CNRS, IRISA\protect\\ F-35000 Rennes, France; and while G. Blanchard
was with the University of Potsdam, Germany.} 
\qquad \hfill remi.gribonval@inria.fr\\
Gilles Blanchard 
\thanks{Universit\'e Paris-Saclay, CNRS,  Inria, Laboratoire de math\'ematiques d'Orsay,
  F-91405, Orsay, France.}
\qquad \hfill gilles.blanchard@universite-paris-saclay.fr\\
Nicolas Keriven\thanks{CNRS, GIPSA-lab, UMR 5216, F-38400 Saint-Martin-d'H\`eres, France.}
\hfill 
nicolas.keriven@gipsa-lab.grenoble-inp.fr\\
Yann Traonmilin\thanks{CNRS, Univ. Bordeaux, Bordeaux INP,  IMB, UMR 5251, F-33400 Talence, France.}
\hfill
yann.traonmilin@math.u-bordeaux.fr
 }

\maketitle

\begin{abstract}
We describe a general framework --{\em compressive statistical learning}-- for resource-efficient large-scale learning: the training collection is compressed in one pass into a low-dimensional {\em sketch} (a vector of random empirical generalized moments) that captures the information relevant to the considered learning task. A near-minimizer of the risk is computed from the sketch through the solution of a nonlinear least squares problem. We investigate sufficient sketch sizes to control the generalization error of this procedure. The framework is illustrated on compressive PCA, compressive clustering, and compressive Gaussian mixture Modeling with fixed known variance. The latter two are further developed in a companion paper.
\end{abstract}


\section{Introduction}

Large-scale machine learning faces a number of fundamental computational challenges, triggered both by the high dimensionality of modern data and the increasing availability of very large training collections. Besides the need to cope with high-dimensional features extracted from images, volumetric data, etc., 
a key challenge is to develop techniques able to fully leverage the information content and learning opportunities opened by large training collections of millions to billions or more items, with controlled computational resources. 

Such training volumes can severely challenge traditional statistical learning paradigms based on batch empirical risk minimization. Statistical learning offers a standardized setting where learning problems are expressed as the optimization of an expected loss, or risk, $\Risk(\Prob,\hyp) := \Exp_{\Sample\sim\Prob}{\loss(\Sample,\hyp)}$ over a parameterized family of hypotheses $\HypClass$ (where $\Prob$ is the probability distribution of the training collection).
This risk is empirically estimated on a training collection, and parameters that empirically minimize it are seeked, possibly with some regularization. Empirical minimization typically requires access to the whole training collection, either in batch mode or iteratively with one or more passes of stochastic gradient. This can become prohibitively costly when the collection is large and each iteration has non-negligible cost. An alternative is to sub-sample the collection, but this may come at the price of neglecting some important items from the collection. Besides online learning \citep[e.g.][]{Mairal2010}, sampling techniques such as coresets \citep{Feldman2011} or Nystr{\"o}m's method \citep[e.g.][]{Rudi:2015ud} have emerged to circumvent computational bottlenecks and preserve the ability to exploit latent information from large collections. 

Can we design an \rev{alternative} learning framework, with the ability to compress the training collection before even starting to learn? We advocate a possible route, {\em compressive statistical learning}, which is inspired by the notion of \emph{sketching} and is endowed with favorable computational features especially in the context of the streaming and distributed data model \citep{Cormode2011} (see Section \ref{sec:related_work}). Rooted both in the generalized method of moments \citep{Hall2005} and in compressive sensing \citep{FouRau13}, it leverages techniques from kernel methods such as kernel mean embeddings \citep{Gretton2007, Sriperumbudur2010} and random Fourier features \citep{Rahimi2007} to obtain innovative statistical guarantees.

As a trivial example,
assume $\sample,\hyp$ belong to $\RR^\sampleDim$, and
consider the squared loss $\loss(\sample,\hyp)=\norm{\sample-\hyp}^2$, whose risk minimizer
is $\Exp[\Sample]$.
In this specific example, keeping only the $\sampleDim$ empirical
averages of the coordinates of $\Sample$ is obviously sufficient. 
The vision developed in this paper is that, for certain learning problems, all the necessary information can be captured in a {\em sketch}: a vector of empirical (generalized) moments of the collection that captures the information relevant to the considered learning task. Computing the sketch is then feasible in one pass, and a near-minimizer of the risk can be computed from the sketch with controlled generalization error.

This paper is dedicated to show how this phenomenon can be generalized: roughly speaking, can the sketch size be taken to be proportional to the number of ``intrinsic parameters'' of the learning task? Another fundamental requirement for the sketching operation is to be online. When recording the training collection, it should be possible to update the sketch at almost no additional cost. The original training collection can then be discarded and learning can be performed from the sketch only, potentially leading to privacy-preservation. As shown in the companion paper \citeparttwo, a sketching procedure based on random generalized moments meets these requirement for clustering and Gaussian mixture estimation.

\subsection{Inspiration from compressive sensing}

Another classical example of learning task is (centered) Principal Component Analysis (PCA). 
In this setting, $\sample \in \RR^{\sampleDim}$, $\hyp$ is an arbitrary linear subspace of dimension $\PCAdim$, and the loss is $\loss(\sample,\hyp) = \norm{\sample-P_{\hyp} \sample}_{2}^{2}$ with $P_{\hyp}$ the orthogonal projector onto $\hyp$. The
matrix of second moments $\Cov_{\Prob} := \Exp_{\Sample \sim \Prob} \Sample\Sample^{T}$ is known to summarize all the information needed to select the best subspace for a training collection. It thus constitutes a natural sketch (of finite dimension $\sampleDim^{2}$) of the training set. 

A much smaller sketch can in fact be computed. Results from compressive sensing and low-rank matrix completion \citep{FouRau13} allow to compress the matrix of second moments to a sketch of dimension of the order of $\PCAdim \sampleDim$ (much smaller that $\sampleDim^{2}$ when $\PCAdim \ll \sampleDim$) from which the best rank-$\PCAdim$ approximation to $\Cov_{\Prob}$ can be accurately estimated (this rank-$k$ approximation allows to calculate the PCA with appropriate learning guarantees, as we will see in Section~\ref{sec:CompressivePCA}).
This compression operation is made using random linear projections on $\Cov_{\Prob}$, which can be seen as random second order moments of the training collection. 

We propose to generalize such a sketching procedure to arbitrary random generalized moments. Given a learning task and training collection, we study the following questions:
\begin{itemize}
 \item How can we perform learning from a sketch of the training collection?
 \item What statistical learning guarantees can we obtain with such a procedure?
 \end{itemize}
 \subsection{Contributions}
 \label{pfffhhblub}
In this paper, we present a general compressive learning framework.  
\begin{itemize}
 \item We describe a generic {\bf sketching mechanism} with random generalized moments and provide a theoretical learning procedure from the sketched data. 
 \item We derive general {\bf learning guarantees} for sketching with random generalized moments.
 \end{itemize}
In the companion paper \citeparttwo, we exploit this framework to establish statistical learning guarantees for  compressive clustering and compressive Gaussian mixture estimation.
We conclude this paper by briefly discussing the potential impact of the proposed framework and its extensions in terms of privacy-aware learning and of the insight it may bring on the information-theoretic properties of certain convolutional neural networks.

\subsection{Related work}\label{sec:related_work}

\paragraph{Sketching and streaming methods.}  \emph{Sketches}  are closely linked with the development of \emph{streaming methods} \citep{Cormode2011}, in which data items are seen once by the user then discarded. A sketch is a small summary of the data seen at a given time, that can be queried for a particular piece of information about the data. As required by the streaming context, when the database is modified, e.g. by inserting or deleting an element, the subsequent update of the sketch must be very fast. In practice, sketches are often applied in various contexts where the data are stored in multiple places. In this heavily distributed framework, a popular class of sketches is that of \emph{linear} sketches, i.e. structures such that the sketch of the union of two databases is the sum of their sketches -- then the sketch of a database distributed over several parts is simply the sum of all their sketches. The sketch presented in this work is indeed a linear sketch (when considered without the normalization constant $1/\nSamples$) and as such, updates operations are excessively simple and fast. Sketches have been used for a large variety of operations \citep{Cormode2011} such as the popular detection of heavy-hitters \citep{Cormode2005,Cormode2009}. Closer to our framework, sketches have been used to approximately maintain histograms \citep{Thaper:2002kf} or quantiles \citep{Gilbert02howto}, however these methods are subject to the well-known curse of dimensionality and are unfeasible even in moderate dimension.

\paragraph{Learning in a streaming context.} Various learning algorithms have also been directly adapted to a streaming context. Examples include the Expectation-Maximization algorithm \citep{Andrieu2003,Cappe2009}, the $k$-means algorithm \citep{Guha2000,Ailon2009}, or Principal Component Analysis \citep{Ghashami2016}. In each case, the result of the algorithm is updated as new data becomes available. However these algorithms do not fully benefit from the many advantages of sketches. Sketches are simpler to merge in a distributed context, update operations are more immediate, and the learning step can be delocalized and performed on a dedicated machine.

\paragraph{Coresets.}
Another popular class of structures that summarize a database for learning is called \emph{coresets}. Coresets were initially developed for $k$-means \citep{Har-Peled2004} or, more generally, subspace approximation \citep{Feldman2010,Feldman2011} and also applied to learning Gaussian Mixture Models \citep{Feldman2011a,Lucic2017}. In a sense, the philosophy behind coresets is situated halfway between sketches and streaming learning algorithms. Like the sketching approaches, coresets methods construct a compressed representation of the database (or ``coreset''), but are somehow closer to already approximately performing the learning task. For instance, the coreset described in \citep{Frahling2005} already incorporates steps of Lloyd's $k$-means algorithm in its construction. Similar to the $k$-means++ algorithm \citep{Arthur2007}, many coresets have been developed as (weighted) adaptive subsampling of the data \citep{Feldman2011a,Lucic2017}. 

\paragraph{Linear sketches {\em vs} Coresets.} It is in general difficult to compare sketching and coresets methods (including the sketching method presented in this paper) in terms of pure performance or theoretical guarantees, since they are very different approaches that can be more or less adapted to certain contexts. We can however outline some differences. Unlike sketches, coresets are not specifically build for the streaming context, and they may require several passes over the data. Nevertheless they can still be adapted to streams of data \citep[as described e.g. in][]{Har-Peled2004,Feldman2011,Lucic2017} by using a merge-and-reduce hierarchical strategy: for each batch of data that arrives sequentially, the user builds a coreset, then groups these coresets and builds a coreset of coresets, and so on. This update method is clearly less direct than updating a linear sketch, and more importantly the user must balance between keeping many coresets and letting the size of the overall summary grow with the number of points in the database, or keeping only highest-level coresets at the cost of losing precision in the theoretical guarantees each time the height of the hierarchical structure increases. As a comparison, the sketch presented in the companion paper \citeparttwo  for $k$-means  %
does not have these limitations: like with any linear sketch, updates are totally independent of previous events, and for a fixed sketch size the ability to perform the learning task strictly increases with the number of points. 

\paragraph{Generalized Method of Moments and Compressive Sensing.}
The methodology that we employ to develop the proposed sketching framework is similar to a Generalized Method of Moments (GeMM) \citep{Landau1987,Hall2005}: the parameters $\Param$ of a model are learned by matching a collection of theoretical generalized moments from the distribution $\Prob_\Param$ with empirical ones from the data. GeMM is often seen as an alternative to Maximum Likelihood estimation, to obtain different identifiability guarantees \citep{Belkin2010a,Hsu2013,Anderson2013} or when the likelihood is not available. Traditionally, a finite number of moments is considered, but modern developments give guarantees when an infinite (integral) number of generalized moments are available \citep{Carrasco2000,Carrasco2014}, in particular generalized moments associated to the (empirical) characteristic function \citep{Carrasco:2002ti,Feuerverger:1977bc}. Our point of view is slightly different: we consider the collection of moments as a \emph{compressed} representation of the data and as a means to achieve a learning task. 

Compared to the guarantees usually obtained in GeMM such as consistency and efficiency of the estimator $\hat{\Param}$, the results that we obtain are more akin to Compressive Sensing and Statistical Learning. For instance, when learning Gaussian Mixture Models, %
we prove in the companion paper \citeparttwo that learning is robust to modeling error (the true distribution of the data is not exactly a GMM but close to one), which is generally overlooked \rev{in} GeMM. In the proof technique, this is done by replacing the so-called ``global identifiability condition'', (i.e. injectivity of the moment operator), which is a classical condition in GeMM but is already difficult to prove and sometimes simply assumed by practitioners \citep[see][p. 2127]{Newey1994} by the strictly stronger Lower Restricted Isometry Property (LRIP) 
from the Compressive Sensing literature \citep{MR2241189,stablesigrecovery-CandesRombergTao-2006,Baraniuk:2007aa,FouRau13}. This is achieved by considering {\em random} feature moments (related to random features \citep{Rahimi2007,Rahimi2009,Bach:2015ux} and kernel mean embeddings \citep{Sriperumbudur2010}), so in a sense the resulting Compressive Statistical Learning framework could be considered as a \emph{Method of Random Feature Moments}.
While the LRIP is reminiscent of certain kernel approximation guarantees with random features \citep[see e.g.][]{Sriperumbudur:2015to,Bach:2015ux}, it is in fact of a different nature, and none seems to be a direct consequence of the other. 

\subsection{Outline}

Section~\ref{sec:general_framework} describes our general framework for compressive statistical learning. We define here statistical learning guarantees, introduce the required notions and state our general Theorem on statistical learning guarantees for compressive learning.  
An important concept is the notion of Lower Restricted Isometry Property (LRIP) using the notion of a model set (a set of ``simple'' probability distributions) which is futher discussed in Section~\ref{sec:choosemodel}.
To illustrate the proposed framework, we detail in Section~\ref{sec:CompressivePCA} 
a procedure for Compressive PCA, where we do not intend to match the latest developments in the domain of PCA such as stochastic and incremental PCA~\citep{Arora2012,Balsubramani2013}, \rev{ kernel PCA using Nyström sampling \citep{Sterge:2020kpca} or random features \citep{Ullah:2018nips,Sriperumbudur:2020approx};}
but rather to give a first illustration.
Generic techniques to establish the LRIP property for sketches of controlled size are described in Section~\ref{sec:ChoiceSketch}.
In the companion paper \citeparttwo, we specify a sketching procedure and state the associated learning guarantees for compressive clustering and compressive Gaussian mixture estimation.
We discuss in Section~\ref{sec:future} possible extensions of the proposed framework as well as the insight it may bring on the information flow across one layer of a convolutive neural network with average pooling.
Finally, all proofs are stated in the Appendix.
\enlargethispage{0.5cm}
\section{A general compression framework for statistical learning}\label{sec:general_framework}
This section is dedicated to the introduction of our compressive learning framework. 

\subsection{Statistical learning}

Statistical learning offers a standardized setting where many learning problems (supervised or unsupervised) can be expressed as the optimization of an expected risk over a parameterized family of functions. Formally, we consider a training collection $\dataset=\{\sample_{i}\}_{i=1}^{\nSamples} \in \SampleSpace^{\nSamples}$  drawn i.i.d. from a probability distribution $\Prob$ on the measurable space $(\SampleSpace,\mathfrak{Z})$. In our examples, $\SampleSpace = \RR^\sampleDim$ is endowed with the
Borel $\sigma$-algebra $\mathfrak{Z}$. One wishes to select a hypothesis $\hyp$ from a hypothesis class $\HypClass$ to perform the task at hand. How well the task can be accomplished with the hypothesis $\hyp$ is typically measured through a {\em loss function} $\loss: (\sample,\hyp) \mapsto \loss(\sample,\hyp) \in \RR$ and the {\em expected risk} associated to $\hyp$:
\[
\Risk(\Prob,\hyp) := \Exp_{\Sample \sim \Prob}\ \loss(\Sample,\hyp)\,,
\]
where (here and in the sequel) we will always assume that we restrict our attention to probability
  distributions $\Prob$ such that $x\mapsto \ell(x,h)$ is measurable and $\Prob$-integrable for all $h\in \HypClass$.
  In the idealized learning problem, one selects a function $\hyp^{\star}_{\Prob}$ that minimizes the expected risk (we will assume existence of this minimum for a simpler presentation, although most
    statements to come can be transformed if needed using a sequence of approximate minimizers)
\begin{equation}\label{eq:BestHyp}
\hyp^{\star}_{\Prob} \in\arg\min_{\hyp \in \HypClass} \Risk(\Prob,\hyp).
\end{equation}
We will use the shorthand $\hyp^{\star}$ for $\hyp^{\star}_{\rev{\Prob}}$ whenever
there is no ambiguity from the context.
In practice one has no access to the true risk $\Risk(\Prob,\hyp)$ since the expectation with respect to the underlying probability distribution, $\Exp_{\Sample\sim\Prob} [\cdot]$, is unavailable. Instead, methods such as \emph{empirical risk minimization (ERM)} produce an estimated hypothesis  $\hat{\hyp}$ from the training dataset $\dataset$ by minimizing the risk
$\Risk(\empProb,\cdot)$ (or a regularized version) associated to the  {\em empirical probability distribution} $\empProb := \tfrac{1}{\nSamples}\sum_{i=1}^{\nSamples} \delta_{\sample_{i}}$ of the training samples. 
One expects to  produce, with high probability at least $1-\probLevel$ on the draw of the training set, the  bound on the excess risk
\begin{equation}\label{eq:ControlRiskHyp}
   \Risk(\Prob,\hat{\hyp}) - \Risk(\Prob,\hyp^\star) \leq  \eta_{\nSamples} = \eta_{\nSamples}(\probLevel),
\end{equation}
where $\eta_{\nSamples}$ typically decays as $1/\sqrt{n}$ or better.
We will use the following running examples.
\paragraph{Examples:}
\begin{itemize}
\item {\bf PCA}: as stated in the introduction, the loss function is  $\loss(\sample,\hyp) = \norm{\sample-P_{\hyp} \sample}_{2}^{2}$ where $P_{\hyp}$ is the orthogonal projection onto the subspace hypothesis $h$ of prescribed dimension $k$.
\item {\bf $k$-means clustering}: each hypothesis corresponds to a set of $k$ candidate cluster centers, $\hyp = \set{c_1,\ldots,c_k}$, and the loss is defined by the $k$-means cost $\loss(\sample,\hyp) = \min_{1 \leq l \leq k} \norm{\sample-c_{l}}_{2}^{2}$. The hypothesis class $\HypClass$ may be further reduced by defining constraints on the considered centers (e.g., in some domain, or as detailed in the companion paper \citeparttwo with some separation between centers). %
\item {\bf Gaussian Mixture Modeling}: each hypothesis $\hyp$ corresponds to the collection of weights, means and variances of a mixture of $k$ Gaussians, whose probability density function is denoted $\Prob_\hyp(\sample)$. The loss function is based on the maximum likelihood $\ell(\sample,h)=-\log \Prob_h(\sample)$. %

\end{itemize}

\subsection{Compressive learning} 

Our aim, and one of the major achievements of this paper, is to control the excess risk~\eqref{eq:ControlRiskHyp} using an estimate $\hat{\hyp}$ {\em obtained from the sole knowledge of a sketch of the training collection}. As we will see, the resulting philosophy for large-scale learning is, instead of addressing an ERM optimization problem of size proportional to the number of training samples,  to first compute a sketch vector {\em of size driven by the complexity of the task}, then to address a nonlinear least-squares optimization problem associated to the {\em Generalized Method of Moments (GeMM)} on this sketch.

 Taking its roots in compressive sensing \citep{MR2241189,stablesigrecovery-CandesRombergTao-2006,FouRau13} and the generalized method of moments  \citep{Landau1987,Hall2005}, but also on kernel mean embeddings \citep{Smola2007,Sriperumbudur2010}, random features \citep{Rahimi2007,Rahimi2009,Bach:2015ux}, and streaming algorithms \citep{Gilbert02howto,Cormode2005,Cormode2011}, \emph{compressive learning} 
relies on the choice of a \rev{measurable} (nonlinear) \emph{feature function}
 $\SketchingOperator:
 \SampleSpace \rightarrow  \RR^{\nMeasures}\ \text{or}\ \CC^{\nMeasures}$ and has two main steps:
\begin{enumerate}
\item Compute generalized empirical moments using the feature function on the training collection to summarize it into a single \emph{sketch vector}
\begin{equation}\label{eq:GenericSketching}
\vy := \Sketch(\dataset) := \frac{1}{n} \sum_{i=1}^{\nSamples}\SketchingOperator(\sample_i) \in \RR^{\nMeasures} \text{or}\ \CC^{\nMeasures};
\end{equation}
\item Produce a hypothesis from the sketch using an appropriate learning procedure:
$\hat{\hyp} = \Learn(\vy)$.
\end{enumerate}
Overall, the goal is to design the %
\rev{feature} function $\SketchingOperator(\cdot)$ and the learning procedure $\Learn(\cdot)$ given a learning task (i.e., a loss function) such that the resulting hypothesis $\hat{\hyp}$ has controlled excess risk~\eqref{eq:ControlRiskHyp} (if $\SketchingOperator$ is drawn at random according to some specification, we want
\eqref{eq:ControlRiskHyp} to hold with high probablity also with respect to the draw of $\SketchingOperator$).
To anticipate 
 \rev{on what will be developed in Section~\ref{sec:explicitriskproxy} (notably Eq.~\eqref{eq:GenericRiskProxy})}, 
 let us mention that learning from a sketch will take the form of a minimization problem 
\begin{equation}\label{eq:DefRiskProxy}
\hat{\hyp} \in \arg\min_{\hyp \in \HypClass} \proxyRisk(\vy,\hyp)
\end{equation}
 where in a sense $\proxyRisk(\vy,\cdot)$ will play the role of a proxy for the empirical risk $\Risk(\empProb,\cdot)$.

\paragraph{Trivial examples.} 
\begin{itemize}
 \item  Estimation of the mean: Assume $\sample,\hyp$ belong to $\RR^\sampleDim$, and
consider the squared loss $\loss(\sample,\hyp)=\norm{\sample-\hyp}^2$, whose risk minimizer is $\Exp[\Sample]$.
In this specific example, it is obviously sufficient to keep only the $\sampleDim$ empirical
averages of the coordinates of $\Sample$, i.e., to use $\SketchingOperator(\sample) := \sample$.
 \item PCA: As the  
 principal components are calculated from the eigenvalue decomposition of the matrix of second moments of the samples, we can simply 
 use $\SketchingOperator(\sample) := \sample\sample^{T}$.
\end{itemize}

A less trivial example is {\em Compressive PCA}. Instead of estimating the full matrix $\Cov_{\Prob} ~=~\Exp_{X \sim \Prob} XX^{T}$, of size $\sampleDim \times \sampleDim$, it is known that computing $\nMeasures$ random gaussian linear measurements of this matrix makes it possible to manipulate a vector $\vy$ of dimension $\nMeasures$ of the order of $k\sampleDim$ from which one can accurately estimate the best rank-$k$ approximation to $\Cov_{\Prob}$, that gives the $k$ first principal components. Nuclear norm minimization is typically used to produce this low rank approximation given the vector $\vy$. We will describe this procedure in details in Section~\ref{sec:CompressivePCA} as a first illustration of our framework.
In the companion paper \citeparttwo, for the more challenging examples of {\em Compressive $k$-means} and {\em Compressive Gaussian Mixture Modeling}, we provide a feature function $\SketchingOperator$ and a method ``$\Learn$'' (based on a specific proxy~\eqref{eq:DefRiskProxy} corresponding to a non-convex least-squares minimization) that leads to a control of the excess risk. 

As described below, these results are achieved by establishing links with the formalism of linear inverse problems and low complexity recovery (i.e., sparse/structured vector recovery, low-rank matrix recovery) and extending theoretical tools to the setting of compressive statistical learning.

\subsection{Compressive learning as a linear inverse problem} 
The most immediate link with linear inverse problems is the following. The sketch vector $\vy$ can be seen as a \emph{linear} function of the empirical probability distribution $\empProb := \tfrac{1}{\nSamples}\sum_{i=1}^{\nSamples} \delta_{\sample_{i}}$: 
\[
 \vy = \rev{\Sketch(\dataset) = } \frac{1}{\nSamples} \sum_{i=1}^{\nSamples} \SketchingOperator(\sample_{i})  = \SketchingOperatorProb(\empProb),
\]
where $\SketchingOperatorProb$ is a linear operator from the set of distributions
$\Prob$ such that $\SketchingOperator$ is integrable with respect to $\Prob$, to $\RR^m$ (or $\CC^{\nMeasures}$), defined by 
\begin{equation}\label{eq:SketchingOperatorProbDef}
 \SketchingOperatorProb(\Prob)  :=\Exp_{\Sample \sim \Prob}\SketchingOperator(\Sample).
\end{equation}
This is linear in the sense that\footnote{One can extend $\SketchingOperatorProb$ to a linear operator on the space of finite signed measures such that $\SketchingOperator$ is integrable, see Appendix~\ref{sec:FiniteSignedMeasures}.} $\SketchingOperatorProb(\theta \Prob+(1-\theta)\Prob') = \theta\SketchingOperatorProb(\Prob)+(1-\theta)\SketchingOperatorProb(\Prob')$ for any $\Prob,\Prob'$ and $0 \leq \theta \leq 1$. 

Since for large $n$ we should have $\SketchingOperatorProb(\empProb) \approx \SketchingOperatorProb(\Prob)$, the sketch $\vy$ can be viewed as a noisy linear observation of the underlying probability distribution $\Prob$. 
This viewpoint allows to formally leverage the general methodology of linear inverse problems to produce a hypothesis from the sketch $\vy$.

Conceptually, we construct the learning-from-sketch procedure $\hat{\hyp} = \Learn(\vy)$ in two steps:
\begin{itemize}
 \item Define a so-called {\em decoder} $\Delta$ that finds a probability distribution $\estProb$ given a sketch $\vy$:
 \begin{equation*}
  \estProb = \Delta[\vy];
 \end{equation*}
 \item Find a best hypothesis from this estimate: 
 \begin{equation}\label{eq:defEstHyp}
  \hat{\hyp} %
  \in \arg\min_{h \in \HypClass} \Risk(\estProb,h).
\end{equation}
\end{itemize}
As a first coarse analysis of this scheme, notice that if the decoder step is such that
a uniform approximation between the risk of $\estProb$ and of $\Prob$ holds:
\begin{equation}\label{eq:UnifControlRisk}
\sup_{\hyp \in \HypClass} | \Risk(\Prob,\hyp)-\Risk(\estProb,\hyp)  | \leq \tfrac12 \eta_{\nSamples},
\end{equation}
then we will be able to control the excess risk~\eqref{eq:ControlRiskHyp} -- our goal.
Indeed, using~\eqref{eq:defEstHyp} and the triangle inequality, it is easy to show that~\eqref{eq:UnifControlRisk} directly implies~\eqref{eq:ControlRiskHyp}. (We will see later that~\eqref{eq:UnifControlRisk} can be too coarse and will
introduce a more refined analysis based on excess risks in Section~\ref{se:improvedanalysis}.) In a way, this is very similar to ERM except that instead of using the empirical risk $\Risk(\empProb,\cdot)$, we use an estimate of the risk $\Risk(\estProb,\cdot)$ where $\estProb$ is deduced directly from the sketch $\vy$.  

\begin{remark} \label{rk:nodensityestimation}
At first sight, the above conceptual view may wrongly suggest that compressive learning replaces statistical learning with the  much more difficult problem of {\em non-parametric density estimation}. Fortunately, as we will see, this is not the case, thanks to the fact that our objective is never to accurately estimate $\Prob$ in the standard sense of density estimation as, e.g., in~\citep{BERTIN:2009:HAL-00381984:1}, but only to accurately estimate the risk $\Risk(\Prob,\cdot)$. On practical examples, a natural decoder will be based on best moment matching over a parametric family of probability distributions, which will be expressed more directly as the minimization of a proxy for the risk \eqref{eq:DefRiskProxy}, cf Section~\ref{sec:explicitriskproxy}.
\end{remark}

\subsection{Statistical learning guarantees: a first control of the excess risk}\label{sec:control_excess_risk}

In this section, for simplicity we first focus on how to establish uniform control of the risks of the form~\eqref{eq:UnifControlRisk} using general results from linear inverse problems (we shall
introduce in the next section a sharper but also slightly more
notation-heavy analysis).
To leverage the links between compressive learning and general inverse problems, we further notice that
\(
\sup_{\hyp \in \HypClass} | \Risk(\Prob,\hyp)-\Risk(\Prob',\hyp)  |
\)
can be viewed as a metric on probability distributions. Given a class $\GClass$ of measurable
functions $f: \SampleSpace \to \RR\ \text{or}\ \CC$, we use the following notation throughout this work:
\begin{equation}
\label{eq:DefFNorm}
\normfclass{\Prob-\Prob'}{G} := \sup_{f \in \GClass }\abs{\Exp_{\Sample \sim \Prob} f(\Sample)-\Exp_{\Sample' \sim \Prob'}f(\Sample')},
\end{equation}
which defines a semi-norm on the space of finite signed measures (see Appendix~\ref{sec:FiniteSignedMeasures}) on $(\SampleSpace,\mathfrak{Z})$ %
 such that all $f \in \rev{\GClass}$ are integrable. In order to be explicit about the integrability
  assumptions in the results to come, we will call this space
  the set of  $\GClass$-integrable finite signed measures (resp. probability distributions, when appropriate).

With this notation,
we have $\sup_{\hyp \in \HypClass} \abs{ \Risk(\Prob,\hyp)-\Risk(\Prob',\hyp)} = \normlossf{\Prob-\Prob'}{}$ where
\begin{equation}\label{eq:DefRiskNorm}
\LossClass(\HypClass) := \set{\loss(\cdot,\hyp): \hyp \in \HypClass}.  
\end{equation}
We will usually abbreviate the latter notation by dropping the dependence on $\HypClass$, considered fixed. 
The desired guarantee~\eqref{eq:UnifControlRisk} then reads $\normloss{\Prob-\Delta[\vy]}{} \leq \eta_{n}/2$.

In the usual context of linear inverse problems, producing an accurate estimate 
from noisy underdetermined linear observations 
requires some ``regularity''  assumption.
Such an assumption often takes the form of a ``low-dimensional'' model set 
that the quantity to estimate 
is close to. 
\begin{example}
In the case of sparse vector recovery (respectively low-rank matrix recovery), one wishes to estimate $\vx \in \RR^{n}$ (resp. $\mX \in \RR^{n \times n}$) from $\vy \approx \mA \vx$ (resp. $\vy \approx \mA \mathrm{vec}(\mX)$). Guarantees are achieved when $\vx$ is close to the set of $k$-sparse vectors (resp. when $\mX$ is close to the set of rank-$r$ matrices).
\end{example}

Similarly here, estimating $\Prob$ from $\vy \approx \SketchingOperatorProb(\Prob)$ may require considering some model set $\Model$, whose choice and definition will be discussed in Section~\ref{sec:choosemodel}. 
\begin{remark}
While in classical compressive sensing the model set plays the role of prior knowledge on the data distribution that completes the observations, in the examples considered here we will often obtain {\em distribution free} excess risk guarantees using models {\em derived from the loss function}. 
\end{remark}

Given a model set\footnote{We will always assume that the models $\Model$ under consideration are such that loss and feature functions are integrable with respect to any distribution belonging to $\Model$, i.e. $\Model$ is  both $\LossClass$-integrable and $\set{\SketchingOperator}$-integrable,
using the terminology introduced after~\eqref{eq:DefFNorm}.}  $\Model$ that plays the role of regularizer, and a sketching operator $\SketchingOperatorProb$, an ``ideal'' decoder $\Decoder$ should be robust to two different sources of error: the distribution of the data $\Prob$ generally does not belong to $\Model$ but is ``close'' to it, introducing some \emph{modelling error}, and the empirical sketch is used instead of the true generalized moments, which adds some \emph{noise}. Generalizing early formulations for sparsity-regularized inverse problems, a decoder robust to both noise and modelling error is usually reffered to as \emph{instance optimal} \citep{Cohen_2009, Bourrier2014}. Mathematically, this can be expressed as: for any distribution $\Prob$, any draw of the training samples from $\Prob$ (embodied by the empirical distribution $\empProb$), with $\vy = \SketchingOperatorProb(\empProb)$ and $\estProb=\Decoder[\SketchingOperatorProb(\empProb)]$
\begin{equation}
\label{eq:InstanceOptimality}
\normloss{ \estProb-\Prob}{}
\lesssim
d(\Prob,\Model) + \norm{\SketchingOperatorProb(\Prob)-\SketchingOperatorProb(\empProb)}_{2}
\end{equation}
where $\lesssim$ hides multiplicative constants, and $d(\cdot,\Model)$ is some measure of distance to the model set $\Model$. In the rest of the paper, we refer to this first term as ``bias''. A significant part of later sections will be devoted to the control of the bias and the choice of a good model set.
Proving that a decoder satisfies~\eqref{eq:InstanceOptimality} ultimately serves to establish bounds such as~\eqref{eq:UnifControlRisk} to control the excess risk.

It turns out that general results from abstract linear inverse problems \citep{Bourrier2014} can be adapted to already characterize the {\em existence} of a decoder satisfying %
property \eqref{eq:InstanceOptimality}. 
By \citep[Section IV-A]{Bourrier2014}, if a decoder with the above property exists then a so-called {\em lower Restricted Isometry Property (LRIP)}
must hold: there is a finite constant $C_{\mathcal{A}} < \infty$ such that
\begin{equation}
\label{eq:lowerRIP}
\normloss{\mProb'-\mProb}{} \leq C_{\mathcal{A}} \norm{\SketchingOperatorProb(\mProb')-\SketchingOperatorProb(\mProb)}_{2}\quad \forall \mProb,\mProb' \in \Model.
\end{equation}
Conversely, the LRIP~\eqref{eq:lowerRIP} implies \citep[Theorem 7]{Bourrier2014} that the following decoder (also known as {\em ideal decoder})
\begin{equation}\label{eq:DefIdealDecoder}
\Decoder[\vy]:=\argmin_{\mProb \in \Model} \norm{\SketchingOperatorProb(\mProb)-\vy}_{2},
\end{equation}
which corresponds to best moment matching, is instance optimal, i.e.,~\eqref{eq:InstanceOptimality} holds for any $\Prob$ and $\empProb$, with the particular distance 
\begin{equation}
\label{eq:DefMDist}
\distIOP(\Prob,\Model) := \inf_{\mProb \in \Model} 
\Big\{\normloss{\Prob-\mProb}{} +  C_{\mathcal{A}} \norm{ \SketchingOperatorProb(\Prob)-\SketchingOperatorProb(\mProb)}_{2}\Big\}.
\end{equation}
As a consequence, the LRIP~\eqref{eq:lowerRIP} implies a control of the excess risk achieved with the hypothesis $\hat{\hyp}$ selected with~\eqref{eq:defEstHyp}, where $\estProb = \Delta[\vy]$, as
\begin{equation}\label{eq:MainRiskBound}
\Risk(\Prob,\hat{\hyp}) - \Risk(\Prob,\hyp^{\star}) \leq 4\distIOP(\Prob,\Model) + 4C_{\mathcal{A}} \norm{\SketchingOperatorProb(\Prob)-\SketchingOperatorProb(\empProb)}_{2}
\end{equation}
where we used explicit constants from \citep[Theorem 7]{Bourrier2014}. Note that in the above argument,
  it was never used that the data is distributed i.i.d. from $\pi$. Estimate~\eqref{eq:MainRiskBound} therefore holds
  under this form for {\em any} fixed data sample (in fact, for any empirical distribution $\empProb$, being understood
  that it determines $\hat{\hyp}$) and {\em any} distribution $\Prob$. Of course, the data distributional assumption is useful to control the second term in the bound.

\subsection{Improved excess risk analysis}

\label{se:improvedanalysis}

The analysis of the previous section has the merits of simplicity, generality, and using existing results from linear inverse problems. However it has some limitations, in particular when the bias term $\distIOP(\Prob,\Model)$ is not close to zero. To emphasize this point, we consider a simple example and
  compare the excess risk control (for the same sketched learning procedure) obtained 
  through the general bound~\eqref{eq:MainRiskBound}, to a direct computation specific to this example.

\rev{Consider the problem of estimating the median} of a distribution on $\RR$: we assume $\SampleSpace=\HypClass = \RR$, and consider the
absolute value loss $\loss(\sample,\hyp) = \abs{\sample-\hyp}$, whose risk minimizer under
the distribution $\Prob$ is the median
$\hyp^{\star}=\Med(\Prob)$.
As a sketching operator we take simply $\SketchingOperator(\sample) = \sample$, resulting in the sketch
given by the empirical mean
$\vy = \SketchingOperatorProb(\empProb) = \frac{1}{\nSamples} \sum_{i=1}^{\nSamples} \sample_{i}$.
Finally, as a model we consider the family of 1-point Dirac measures, $\Model=\set{\delta_\sample,\sample \in \RR}
$ (this is the model consisting of all distributions with vanishing optimal
risk, see Section~\ref{se:leastrestrictedmodel} for a more general discussion). Obviously, it then holds with these choices that the ideal decoder given by~\eqref{eq:DefIdealDecoder}
is $\estProb = \Decoder[\vy] = \delta_\vy$, and further $\hat\hyp=\vy$. \\
On the one hand, the excess risk for this sketching/decoding scheme is bounded as follows, by a simple direct calculation, putting $\Moy(\Prob):=\Exp_{\Sample \sim \Prob}[X]$:
\begin{align}
  \Risk(\Prob,\hat{\hyp}) - \Risk(\Prob,\hyp^\star)
  &= \Exp_{\Sample \sim \Prob}\brac{\abs{X-\vy} - \abs{X-\Med(\Prob)}} \nonumber\\
  & %
      \leq \underbrace{\Exp_{\Sample \sim \Prob}[\abs{X-\Moy(\Prob)} - \abs{X-\Med(\Prob)}]}_{=:\mathcal{B}(\Prob)} + \abs{\vy - \Moy(\Prob)}.      \label{eq:directcalc}
\end{align}
On the other hand, it is easy to check that the LRIP~\eqref{eq:lowerRIP} holds, with equality, for 1-point Dirac measures with constant $C_{\SketchingOperatorProb}=1$. If we consider the general bound~\eqref{eq:MainRiskBound}, while
we recover (up to factor 4) the second term above $\norm{\SketchingOperatorProb(\Prob)-\SketchingOperatorProb(\empProb)}_{2} = \abs{\vy - \Moy(\Prob)}$ (of order $\mathcal{O}(1/\sqrt{n})$), the first term in~\eqref{eq:MainRiskBound} is a {\em bias term} which is driven by
  \begin{align}
    \distIOP(\Prob,\Model)
    &:= \inf_{h \in \RR} \Big\{
\sup_{h' \in \RR} \big\lvert \Exp_{\Sample \sim \Prob}\brac{\abs{X-h'}} - |h-h'| \big\rvert +  \abs{ \Moy(\Prob) -h} \Big\}    \nonumber \\
    &\geq  \inf_{h \in \RR}  \Big\{ \Exp_{\Sample \sim \Prob}\brac{\abs{X-h}} +
      \abs{ \Moy(\Prob)-h} \Big\} \nonumber\\
    & \geq \Exp_{\Sample \sim \Prob}\brac{\abs{X-\Moy(\Prob)}} \nonumber \\
    & =  \mathcal{B}(\Prob) + \Risk(\Prob,\hyp^\star).
    \label{eq:firstbound}
  \end{align}
  The inequality in the third line above is obtained by noticing that 
  $\abs{x-\Moy(\Prob)} \leq \abs{x-h}+\abs{\Moy(\Prob)-h}$ for each $x,h \in \RR$.
Hence, using the general bound~\eqref{eq:MainRiskBound} instead of the specific direct calculation,
  we get an additional, unwanted term corresponding to the optimal risk
  $\Risk(\Prob,\hyp^\star)$ which is nonzero as soon as
$\Prob \not\in\Model$, and can become arbitrary large (even if $\mathcal{B}(\Prob)=0$, e.g. if $\Prob$ is symmetric around its mean).

One reason for this lack of sharpness is that the analysis in the previous section concentrated
first on (uniform) control of the risk difference $\normloss{\Prob-\Prob'}{}$, to deduce only as a second step
a control on the excess risk. It is known from the statistical learning literature that it is generally
sharper to directly analyze the excess risk; and correspondingly consider the {\em excess loss} class
\rev{(see, for instance, \citealp[Section 5]{bartlett:2005lo} and \citealp[Section 7]{koltchinskii:2006lo})}:
\begin{equation}\label{eq:DefDLossClass}
\DLossClass(\HypClass):= \LossClass(\HypClass) - \LossClass(\HypClass) 
= \set{g: \sample \mapsto g(\sample) = \loss(\sample,\hyp)-\loss(\sample,\hyp'), \hyp,\hyp' \in \HypClass}.
\end{equation}
However, the risk minimizer $\hyp^{\star}$ depends
on the distribution $\Prob$; for this reason we will consider a {\em family} of excess losses and risks with respect
to some reference hypothesis $\hyp_0$.
\begin{definition}\label{def:excessriskdiv}
The {\em excess risk} relative to
a reference hypothesis $\hyp_0$ is defined as:
\[
  \drisk_{\hyp_0}(\Prob,\hyp) := \Risk(\Prob,\hyp) - \Risk(\Prob,\hyp_0)
  =  \Exp_{X \sim \Prob}[\loss(X,\hyp) - \loss(X,\hyp_0)],
\]
and the associated {\em excess risk divergence} with respect to $\hyp_0$ is:
\begin{equation}\label{lossdiv}
  \divp_{\hyp_0}(\Prob\|\Prob') := \sup_{\hyp\in \HypClass}  \paren{ \drisk_{\hyp_0}(\Prob,\hyp) - \drisk_{\hyp_0}(\Prob',\hyp)}.
\end{equation}
Observe that $\drisk_{\hyp_0}(\Prob,\hyp_0) = 0$ for each $\Prob$, hence the latter quantity is nonnegative (although no absolute value is involved in its definition), but not symmetric in general. Yet, it satisfies an (oriented) triangle inequality: for any $\Prob,\Prob',\Prob''$
\[
  \divp_{\hyp_0}(\Prob\|\Prob') \leq \divp_{\hyp_0}(\Prob\|\Prob'') +  \divp_{\hyp_0}(\Prob''\|\Prob').
\]
It is therefore a {\em hemimetric} (see Definition~\ref{defhemimetric} in the Appendix).
\end{definition}
The excess risk divergence will play a role similar to that of $\normloss{\Prob-\Prob'}{}$, and satisfies
in particular %
\begin{align}
  \sup_{\hyp_0 \in \HypClass} \divp_{\hyp_0}(\Prob\|\Prob')
  & = \sup_{\hyp_0,h \in \HypClass} (\Exp_{X \sim \Prob}[\loss(X,\hyp) - \loss(X,\hyp_0)] - \Exp_{X \sim \Prob'}[\loss(X,\hyp) - \loss(X,\hyp_0)])
  \nonumber \\
    & = \dnormloss{\Prob - \Prob'}{}  \leq 2 \normloss{\Prob - \Prob'}{}.   \label{eq:ineqdivg}
\end{align}
With this setting we have the following result, which can be seen as a refinement of~\eqref{eq:MainRiskBound}.

\begin{theorem}\label{thm:LRIPsuff_excess}
  Consider a loss class $\LossClass(\HypClass)$, a feature function $\SketchingOperator$, and a model set $\Model$ such that every probability distribution $\mProb \in \Model$ is both $\LossClass$-integrable and $\set{\SketchingOperator}$-integrable. 
  Assume that the sketching operator $\SketchingOperatorProb$ associated to $\SketchingOperator$ satisfies the following LRIP inequality:
\begin{equation}
  \label{eq:lowerRIP_excess}
  \dnormloss{\mProb-\mProb'}{} \leq   C_\SketchingOperatorProb \norm{\SketchingOperatorProb(\mProb)-\SketchingOperatorProb(\mProb')}_2 + \eta, \qquad \forall \mProb,\mProb' \in \Model, %
\end{equation}
for some finite constants $C_\SketchingOperatorProb>0$ and $\eta \geq 0$.

Consider  any training collection $\dataset=\{\sample_{i}\}_{i=1}^{\nSamples} \in \SampleSpace^{\nSamples}$, and denote $\empProb := \tfrac{1}{n} \sum_{i=1}^{n} \delta_{\sample_{i}}$.
Define
\begin{align}
\vy &:= \Sketch(\dataset) = \SketchingOperatorProb(\empProb),\label{eq:ThmSketch2}\\
  \estProb \in \Model  \text{ satisfying }
             \norm{\SketchingOperatorProb(\estProb)-\vy}_2 & \leq (1+\nu) \inf_{\mProb \in \Model}
             \norm{\SketchingOperatorProb(\mProb) -\vy}_2 + \eps, 
             & \text{ for some constants } \eps,\nu \geq 0,
             \label{eq:ThmDecoder2}\\
  \hat{\hyp}  \text{ satisfying }
               \Risk(\estProb,\hat\hyp) & \leq \inf_{\hyp \in \HypClass} \Risk(\estProb,\hyp) + \eps',
               & \text{ for some constant } \eps' \geq 0.
                              \label{eq:ThmLearn2}
\end{align}
Then, %
for any probability distribution $\Prob$ that is both $\LossClass$-integrable and $\set{\SketchingOperator}$-integrable:
\begin{equation}\label{eq:MainBoundExcessRisk}
  \forall \hyp_0 \in \HypClass: \;\;
  \drisk_{\hyp_0}(\Prob,\hat{\hyp})
  \leq \distIOPexgen_{\hyp_0}(\Prob,\Model) +
   (2+\nu) C_\SketchingOperatorProb \norm{\SketchingOperatorProb (\Prob)-\SketchingOperatorProb(\empProb)}_2 + \eta + C_{\SketchingOperatorProb} \eps + \eps',
 \end{equation}
where \begin{equation}\label{eq:DefMDist2}
\distIOPexgen_{\hyp_0}(\Prob,\Model) := \inf_{\mProb \in \Model} \paren{\divp_{\hyp_0}(\Prob\|\mProb)
   + (2+\nu)C_\SketchingOperatorProb \norm{\SketchingOperatorProb(\Prob) -\SketchingOperatorProb(\mProb)}_2}.
   \end{equation}
\end{theorem}
Similarly to ~\eqref{eq:MainRiskBound}, the above estimate holds regardless of any distributional
  assumptions on the training collection $\dataset$.
  Nevertheless, estimate~\eqref{eq:MainBoundExcessRisk} is primarily of interest when $\dataset$ is drawn i.i.d. according to $\Prob$ and with $\hyp_0=\hyp^{\star}_{\Prob}$, in which case the left-hand side
  is the excess risk with respect to the optimum risk, which is what one generally aims at
  controlling.
  However, in some situations it may be also helpful to consider
  excess risk with respect to other reference hypotheses $\hyp_0$; this can include situations
  where $\hyp^{\star}_{\Prob}$ itself is not well-defined if the infimum of the
  risk is not attained.\\  
  As compared to%
~\eqref{eq:MainRiskBound},
we observe that this result is more general, as it allows
for a $(\nu,\eps)$-approximate decoder~\eqref{eq:ThmDecoder2}, 
an
$\eps'$-approximate ERM~\eqref{eq:ThmLearn2}, an $\eta$-approximate LRIP condition~\eqref{eq:lowerRIP_excess}; more importantly,
the main bound~\eqref{eq:MainBoundExcessRisk} involves the sharper excess risk divergence rather than the
loss norm $\normloss{.}{}$. It may also be useful to consider $\Prob = \empProb$, to predict the quality compared to the empirical risk minimizer.

Moreover, inequality~\eqref{eq:ineqdivg} implies that the lower LRIP condition~\eqref{eq:lowerRIP}
considered in the previous section implies the relaxed LRIP condition~\eqref{eq:lowerRIP_excess}
(with $\eta=0$, and up to a factor 2 in the constant), so establishing~\eqref{eq:lowerRIP}
is sufficient in order to obtain the improved inequality~\eqref{eq:MainBoundExcessRisk}. 

The proof of Theorem~\ref{thm:LRIPsuff_excess} follows the structure outlined in the previous section, 
but requires to formally extend the result of \citep[Section IV-A]{Bourrier2014}
(leading from the LRIP~\eqref{eq:lowerRIP} to instance optimality~\eqref{eq:InstanceOptimality})
to the case of a hemimetric,
$\eta$-approximate LRIP and $\eps$-approximate decoder. These technical aspects are relegated to
Appendix~\ref{sec:ProofLRIPsuff_excess}.

\paragraph{Discussion:}
\begin{itemize}
\item Computing the sketch~\eqref{eq:ThmSketch2} is highly parallelizable and distributable. Multiple sketches can be easily aggregated and updated as new data become available.
\item As discussed in Remark~\ref{rk:nodensityestimation}, while~\eqref{eq:ThmDecoder2} may appear as a general nonparametric density estimation problem, in all %
the examples considered in this paper and the companion one \citeparttwo, 
it is indeed a nonlinear parametric least-squares fitting problem %
and the existence of the minimizer follows in practice from compactness arguments. 
\begin{itemize}
\item For {\bf Compressive PCA} (Section~\ref{sec:CompressivePCA}) it is a low-rank matrix reconstruction problem. Provably good algorithms to estimate its solution have been widely studied. 
\item For {\bf Compressive $k$-means} and {\bf Compressive Gaussian Mixture Modeling} (cf the companion paper \citeparttwo), the resulting optimization problem has been empirically addressed through the CL-OMPR algorithm \citep{Keriven2015,Keriven2016arxiv}. Algorithmic success guarantees are an interesting challenge. This is however beyond the scope of this paper. We note that the classic (non-compressed) $k$-means problem by minimization of the empirical risk is known to be NP-hard \citep{Garey:1982ks,Aloise2009} and that guarantees for approaches such as K-means++ \citep{Arthur2007} are only in expectation and with a logarithmic sub-optimality factor.
\end{itemize}
\item In Section~\ref{sec:choosemodel} we discuss choices of the model set $\Model$ that are driven by the learning task only and make the minimization problem~\eqref{eq:ThmLearn2} trivial to solve. With these choices the combined solution of~\eqref{eq:ThmDecoder2}-\eqref{eq:ThmLearn2} is explicitly turned into the minimization of a proxy for the risk, as in~\eqref{eq:DefRiskProxy}.
\item The second term in the bound~\eqref{eq:MainBoundExcessRisk} of the excess risk, $\eta_{\nSamples}$, is the empirical estimation error $\norm{ \SketchingOperatorProb(\Prob)-\SketchingOperatorProb(\empProb)}_{2}$.  It is easy to show that it decays as $1/\sqrt{\nSamples}$ when the data is drawn i.i.d. according to $\Prob$,
this will be done explicitly for the considered examples. 
\end{itemize}

For large collection size $n$ drawn i.i.d. according to $\Prob$, the term $\norm{\SketchingOperatorProb(\Prob)-\SketchingOperatorProb(\empProb)}_{2}$ becomes small and~\eqref{eq:MainBoundExcessRisk} shows that compressive learning will benefit from accurate excess risk guarantees provided the model $\Model$ and the feature function $\SketchingOperator$ (or equivalently the sketching operator $\SketchingOperatorProb$) are chosen so that:
\begin{enumerate}
\item 
the LRIP~\eqref{eq:lowerRIP_excess} holds; ideally for a ``small'' value of $\nMeasures$, as we also seek to design compact sketches and, eventually, tractable algorithms to learn from them.
\item  the distance $ \distIOPexgen_{\hyp^{\star}_{\Prob}}(\Prob,\Model)$ is ``small''; this vague notion will be exploited in Section~\ref{sec:choosemodel} below to guide our choice of model set $\Model$, and will be made more concrete on examples. %
\end{enumerate}

 We illustrate the improvement obtained for the bias term with respect to the coarser analysis on the toy
  example of median estimation considered previously. In that setting we have $\hyp^\star_{\Prob} = \Med(\Prob)$, and  for $\mProb = \delta_x \in \Model$:
  \begin{align*}
    \divp_{\hyp^{\star}}(\Prob\|\delta_x)
    & = \sup_{\hyp \in \HypClass} (
    \Exp_{\Sample \sim \Prob}\brac{\abs{X-h} - \abs{X-\Med(\Prob)}}
                                              - (\abs{x - h} - \abs{x-\Med(\Prob)}))  \\
    & = \Exp_{\Sample \sim \Prob}\brac{\abs{X-x}}
      - \Exp_{\Sample \sim \Prob}\brac{\abs{X-\Med(\Prob)}}
      + \abs{x-\Med(\Prob)}
  \end{align*}
  so that
  \begin{align*}
     \distIOPexgen_{\hyp^{\star}}(\Prob,\Model) &:= \inf_{\delta_x \in \Model}
    \paren{\divp_{\hyp^{\star}}(\Prob\|\delta_x)
      + (2+\nu)C_\SketchingOperatorProb \norm{\SketchingOperatorProb(\Prob) - \SketchingOperatorProb(\delta_x)}_2}.\\
    & =    \inf_{x \in \RR } \paren{\Exp_{\Sample \sim \Prob}\brac{\abs{X-x}}
      - \Exp_{\Sample \sim \Prob}\brac{\abs{X-\Med(\Prob)}}
      + \abs{x-\Med(\Prob)} 
      + (2+\nu) \abs{x - \Moy(\Prob)}}.\\
      & \leq \mathcal{B}(\Prob) + \abs{\Med(\Prob) - \Moy(\Prob)}.
  \end{align*}
  The inequality in the third line is obtained by using $x = \Moy(\Prob)$.
  Note that the presence of the last term is unavoidable (since $\abs{x-\Med(\pi_0)} + \abs{x - \Moy(\Prob)} \geq \abs{\Med(\pi_0) - \Moy(\Prob)} $), and that it is still larger than $\mathcal{B}(\Prob)$
     which we recall is the only bias term appearing in the direct calculation~\eqref{eq:directcalc}.
(A situation where it is much larger is the following: assume $\Prob = (\frac{1}{2}+\eps)\delta_{0} + (\frac{1}{2} - \eps) \delta_1$, for $0<\eps<1/2$.                         
  Then $\Med(\Prob) = 0$, $\Moy(\Prob) = \frac{1}{2}-\eps$, $\Exp_{\Sample \sim \Prob} \abs{X-\Moy(\Prob)} = (\tfrac{1}{2}+\eps)(\tfrac{1}{2}-\eps)+(\tfrac{1}{2}-\eps)(\tfrac{1}{2}+\eps) \
= (1+2\eps)(\tfrac{1}{2}-\eps)$, $\Exp_{\Sample \sim \Prob} \abs{X-\Med(\Prob)} = \Exp_{\Sample \sim \Prob} X = \tfrac{1}{2}-\eps$, hence $\abs{\Med(\Prob)-\Moy(\Prob)} = \tfrac{1}{2}\
-\eps$ while $\mathcal{B}(\Prob) = 2\eps(1/2-\eps)$.) In this sense, even using the improved excess risk analysis,
  the general bound~\eqref{eq:MainBoundExcessRisk} can lack some tightness. It is nevertheless much sharper
  than the bound~\eqref{eq:firstbound}, in particular $ \distIOPexgen_{\hyp^{\star}_{\Prob}}(\Prob,\Model)=0$ as soon as $\Moy(\Prob)=\Med(\Prob)$,
  while $ \distIOPexgen_{\hyp^{\star}_{\Prob}}(\Prob,\Model) >0$ in general in this case, see~\eqref{eq:firstbound}.

\section{Task-driven model sets}\label{sec:choosemodel}
An important ingredient of the proposed framework is the model set $\Model$ and we now discuss its choice. 
\rev{If prior knowledge on the data distribution is available, it is of course possible to choose $\Model$ to incorporate such knowledge into compressive statistical learning. } 
However, it is more common in statistical learning to seek ``distribution-free'' statistical guarantees.
\rev{Therefore, it may be more desirable to derive a model set entirely or mostly from the learning task itself, to the extent possible.}

\rev{To begin with, we show that for any model set $\Model$, the abstract two-step learning mechanism (cf steps~\eqref{eq:ThmDecoder2}-\eqref{eq:ThmLearn2} in Theorem~\ref{thm:LRIPsuff_excess}) can be written as the minimization of a more explicit proxy~\eqref{eq:DefRiskProxy} for the empirical risk. This rewriting exploits a partition of $\Model$ into certain submodels $\Model_{\hyp}$ driven by the learning task, i.e. by the loss family $\{\loss(\cdot,\hyp)\}_{\hyp \in \HypClass}$. For certain learning tasks,  we further show that the submodels $\Model_{\hyp}$ and the corresponding proxy~\eqref{eq:DefRiskProxy} have a simple expression provided we choose a ``natural'', task-driven, model set $\Model$.
Finally, for so-called ``compression type'' learning tasks, choosing such a task-driven model set allows to control
the bias term in the excess risk~\eqref{eq:MainBoundExcessRisk} by a function of the optimal risk $\Risk(\Prob,\hyp^{\star})$.}

\subsection{Learning from a sketch without explicit density estimation}
\label{sec:explicitriskproxy}
Consider a family $\Model$ of $\LossClass$-integrable distributions and assume for each $\Prob \in \Model$ the risk admits a minimizer, i.e., there is $\hyp \in \HypClass$ such that $\Risk(\Prob,\hyp) = \inf_{\hyp' \in \HypClass} \Risk(\Prob,\hyp')$. When this holds the model set $\Model$ can be decomposed as $\Model = \cup_{\hyp \in \HypClass} \Model_{\hyp}$ where for each hypothesis $\hyp \in \HypClass$ we define
\begin{equation}\label{eq:DefSimpleModelSetGivenHyp}
\Model_{\hyp} := \set{\Prob \in \Model: \Risk(\Prob,\hyp) \leq \Risk(\Prob,\hyp'), \forall \hyp' \in \HypClass},
\end{equation}
and the hypothesis $\hat{\hyp}$ selected using steps~\eqref{eq:ThmDecoder2}-\eqref{eq:ThmLearn2} in Theorem~\ref{thm:LRIPsuff_excess} (with $\varepsilon'=0$) %
is equivalently obtained as a near-minimizer of the following proxy for the risk
\begin{equation}\label{eq:GenericRiskProxy}
\proxyRisk(\vy,\hyp) := \inf_{\mProb \in \Model_\hyp} \norm{\SketchingOperatorProb(\mProb)-\vy}_{2}\, ,
\end{equation}
in the sense that $\inf_{\mProb \in \Model}\norm{\SketchingOperatorProb(\mProb)-\vy}_{2} = \inf_h \inf_{\mProb \in \Model_\hyp} \norm{\SketchingOperatorProb(\mProb)-\vy}_{2}$. With this expression in hand, it is possible to directly cast the estimation of $\hat\hyp$ as \eqref{eq:DefRiskProxy}. 
To turn this into a concrete proxy for the risk it is helpful to consider a model set $\Model$ such that $\Model_{\hyp}$ has a simple characterization.

\subsection{Choosing a model set: with or without prior knowledge ?}
\label{se:leastrestrictedmodel}
Learning tasks such as maximum likelihood estimation directly involve a natural model set for which $\Model_{\hyp}$ as in \eqref{eq:DefSimpleModelSetGivenHyp} is easily characterized. Consider the loss $\loss(\sample,\hyp) = -\log \Prob_{\hyp}(\sample)$ with $\set{ \Prob_{\hyp}, \hyp \in \HypClass}$ a parameterized family of distributions with $\Prob_{\hyp'} \neq \Prob_{\hyp}$ for $\hyp' \neq \hyp$. 
\rev{In this setting, it is natural to consider the following model set:
 \begin{equation}\label{eq:DefNaturalMLModel}
\ModelML(\HypClass) := \set{\Prob_{\hyp}: \hyp \in \HypClass},
\end{equation}
which is nothing more than a statistical model in the usual sense.
Moreover, up to a constant additive term, it then holds $\Risk(\Prob,h) = \KLdiv{\Prob}{\Prob_h}$, where $\KLdiv{\cdot}{\cdot}$ is the Kullback-Leibler divergence,
so that for any $h \in \HypClass$:
\begin{align*}
  \ModelML_{\hyp} & = \set{\Prob \in \ModelML: \KLdiv{\Prob}{\Prob_\hyp} \leq \KLdiv{\Prob}{\Prob_\hyp'}, \forall \hyp' \in \HypClass}\\
                & =  \set{\Prob_{\hyp''}, \hyp'' \in \HypClass: \KLdiv{\Prob_{\hyp''}}{\Prob_\hyp} \leq \KLdiv{\Prob_{\hyp''}}{\Prob_{\hyp'}}, \forall \hyp' \in \HypClass}\\
  & = \set{\Prob_h},
\end{align*}
see \citep[Chapter 9]{CT91}.
As a conclusion, the proxy~\eqref{eq:GenericRiskProxy} reads $\proxyRisk(\vy,\hyp) =  \norm{\SketchingOperatorProb(\Prob_\hyp)-\vy}_2$.
}

For many other learning tasks, the choice of the model set $\Model$ results from a tradeoff between several needs. 
On the one hand, results from compressed sensing suggest that given a model set $\Model$ that has proper ``low-dimensional'' properties, it is possible to choose a small sketch size $\nMeasures$ and design the sketching operator $\SketchingOperatorProb$ such that the LRIP~\eqref{eq:lowerRIP_excess} holds, and the ideal decoder $\Delta$ in~\eqref{eq:DefIdealDecoder} --- or its relaxed version in~\eqref{eq:ThmDecoder2} --- is guaranteed to stably recover probability distributions in $\Model$ from their compressed version obtained with $\SketchingOperatorProb$. This calls for the choice of a ``small'' model set. On the other hand, and perhaps more importantly, the model set should not be ``too small'' in order to ensure that the obtained control of the excess risk is nontrivial. 

Ideally, in the common case of compression-type tasks, as defined below, the bias term in the excess risk~\eqref{eq:MainBoundExcessRisk} should be small when the true optimum risk is small, and even vanish when the true optimum risk vanishes, i.e. when $\inf_{\hyp \in \HypClass} \Risk(\Prob,\hyp) = 0$. 

  \begin{definition} \label{def:comptypetask}
    We call the learning task a {\em compression-type task} if the loss can be written as
   $\loss(\sample,\hyp) = \divg^{p}(\sample,P_{\hyp}\sample)$, where $d$ is a metric on $\SampleSpace$, $p>0$, and
$P_{\hyp}: \SampleSpace \to \SampleSpace$ is a ``projection function'', i.e.,
     \begin{align}
       P_{\hyp} \circ P_{\hyp} &= P_{\hyp};\label{proj1}\\
       \divg(x,P_{\hyp}x) & \leq  \divg(x,P_{\hyp}x'),\quad \forall x,x' \in \SampleSpace.\label{proj2} 
     \end{align}
 \end{definition}
Typical examples of compression-type tasks are PCA, $k$-means, and $k$-medians. 
For PCA, $P_{\hyp}$ is the orthogonal projector onto subspace $\hyp$.
For $k$-means and $k$-medians, $P_{\hyp}$ maps $\sample \in \SampleSpace = \RR^{\sampleDim}$ to the closest center $c_{i}$ from $\hyp = (c_{1},\ldots,c_{k})$, with ties broken arbitrarily. In other words, given an arbitrary \emph{Voronoi partition} corresponding to $k$ disjoint sets $W_{j}$ such that $\cup_{j} W_{j} = \RR^{\sampleDim}$ and $d(x,c_{j}) = \min_{l} d(x,c_{l})$ for each $x \in W_{j}$, $P_{\hyp}\sample = c_{j}$ if and only if $\sample \in W_{j}$.
Manifold learning tasks where $P_{\hyp}$ is a projection onto a manifold parameterized by $\hyp$ (with ties broken arbitrarily) would also fit under this framework.

For a compression-type task, a natural model set is the family of $\LossClass$-integrable probability distributions 
\begin{equation}
\label{eq:DefLeastRestrictedModel}
\ModelCT(\HypClass) := \cup_{\hyp \in \HypClass} \ModelCT_{\hyp}\qquad \text{where}\quad
\ModelCT_{\hyp} :=  \set{\Prob: \Risk(\Prob,\hyp) = 0}.
\end{equation}
We consider a few examples:
\begin{itemize}
 \item {\bf Compressive PCA}: the model set $\ModelCT(\HypClass)$ consists of all distributions which admit a matrix of second moments of rank at most $k$. Given any $\estProb \in \ModelCT(\HypClass)$, a minimum risk hypothesis according to~\eqref{eq:defEstHyp} is any subspace $\hat{\hyp}$ spanned by eigenvectors associated to the $\PCAdim$ largest eigenvalues of $\Cov_{\estProb}$. More details will be given shortly in Section~\ref{sec:CompressivePCA}.
 \item {\bf Compressive $k$-means or $k$-medians}: the model set $\ModelCT(\HypClass)$ consists of mixtures of $k$ Diracs. Given $\hyp = \set{c_{1},\ldots,c_{k}}$ and any $\estProb = \sum_{\ell=1}^{\PCAdim}\alpha_{\ell}\delta_{c_{\ell}} \in \ModelCT_\hyp$, a minimum risk hypothesis according to~\eqref{eq:defEstHyp} is $\hat{\hyp} = \hyp$.  Since $\SketchingOperatorProb(\delta_{c}) = \SketchingOperator(c)$, the proxy~\eqref{eq:GenericRiskProxy} reads
 \begin{equation}\label{eq:RiskProxyKMeans}
\proxyRisk(\vy,\hyp) = \min_{\alpha \in \Simplex_{\PCAdim-1}}  \norm{\sum_{\ell=1}^k\alpha_\ell \SketchingOperator(c_\ell)-\vy}_2
\end{equation}
with $\Simplex_{\PCAdim-1} := \set{\alpha \in \RR^{\PCAdim}: \alpha_\ell \geq 0; \sum_{\ell=1}^{k}\alpha_{\ell}=1}$ the simplex.

\end{itemize}
For compressive PCA we exhibit in Section~\ref{sec:CompressivePCA} a feature function $\SketchingOperator$ so that $\SketchingOperatorProb$ satisfies the LRIP~\eqref{eq:lowerRIP_excess} with respect to the model set $\Model = \ModelCT(\HypClass)$. The same is done in the companion paper \citeparttwo\ for compressive $k$-means, compressive $k$-medians and compressive Gaussian mixture modeling.

\subsection{Controlling the bias term for compression-type tasks}\label{sec:biascontrol}
\rev{The \emph{bias term} --- defined in~\eqref{eq:DefMDist2}} ---
is a measure of distance to the model set $\Model$. For compression-type tasks, the particular model set $\Model =\ModelCT(\HypClass)$ in \eqref{eq:DefLeastRestrictedModel} was designed so that this \rev{bias term}
vanishes when $\Prob \in \Model$, \rev{and we} %
can further bound the bias term $ \distIOPexgen_{\hyp^{\star}_{\Prob}}(\Prob,\ModelCT(\HypClass))$ 
with an increasing function of the true minimum risk, $\Risk(\Prob,\hyp^{*})$. This leads to recovery guarantees providing {\em distribution-free} excess risk guarantees. Whether this  holds for other learning tasks, or even generically, is a challenging question left to further work. 

The following lemmas allow to obtain an upper bound of the bias term in function of the minimum
  risk in a number of relevant cases.
For a probability distribution $\Prob$ on $\SampleSpace$ and $P_{\hyp}$ as in Definition~\ref{def:comptypetask}, we denote $P_\hyp\Prob$ the  push-forward of $\Prob$ through $P_\hyp$, i.e., the
probability distribution of a random variable $Y = P_\hyp X$ where $X \sim \Prob$. 
Given a loss class $\LossClass(\HypClass)$ we recall that $\hyp^{\star}_{\Prob} = \arg\min_{\hyp \in \HypClass} \Risk(\Prob,\hyp)$.
  \begin{lemma}\label{lem:LemmaBiasTerm}
Consider a compression-type task on the input space $\SampleSpace$. Then 
    \begin{itemize}          
    \item    $\ModelCT_\hyp$ is the set of probability distributions on $X \in \SampleSpace$ such that $X \in \mathcal{E}_{\hyp}:=P_\hyp\SampleSpace$ almost surely.
    With the model set $\ModelCT(\HypClass)$ and the loss class $\LossClass(\HypClass)$, the bias term~\eqref{eq:DefMDist2} satisfies (for any $\hyp_0\in \HypClass$)
      \begin{align}
        \label{eq:biasbound}
        \distIOPexgen_{\hyp_0}(\Prob,\ModelCT(\HypClass)) 
        & \leq
\divp_{\hyp_0}(\Prob\|P_{\hyp_0}\Prob)
   + (2+\nu)C_\SketchingOperatorProb \norm{\SketchingOperatorProb(\Prob) - \SketchingOperatorProb(P_{\hyp_0}\Prob)}_2.
\end{align}
\item If $d^{p}$ is a metric (in particular if $p \leq 1$)
then $\divp_{\hyp}(\Prob\|P_\hyp\Prob)=0$ for any $\hyp \in \HypClass$ and $\LossClass$-integrable distribution $\Prob$.%
  \end{itemize}
  \end{lemma}
 
     \begin{remark}
When $d^{p}$ is not a metric there are $u,v,w \in \SampleSpace$ such that $d^{p}(u,v) > d^{p}(u,w)+d^{p}(w,v)$. The loss $\loss(\sample,\hyp) := d^{p}(\sample,\hyp)$, with $\HypClass := \set{v,w}$, defines a compression-type task with $P_{\hyp}x = \hyp$ for all $\sample \in \SampleSpace$. Set $\Prob = \delta_{u}$. Since $d(u,w) < d(u,v)$ we have $\hyp^{\star}_{\Prob} = w$. We also have for $\hyp = v \in \HypClass$
\begin{align*}
\divp_{\hyp^{\star}_{\Prob}}(\Prob\|P_{\hyp^{\star}_{\Prob}}\Prob)
\geq 
\Delta \mathcal{R}_{w}(\Prob,\hyp)-\Delta \mathcal{R}_{w}(P_{w}\Prob,\hyp)
&= [d^{p}(u,\hyp) - d^{p}(u,w)]-[d^{p}(w,\hyp)-d^{p}(w,w)]\\
&= d^{p}(u,v)-d^{p}(u,w)-d^{p}(w,v) > 0.
\end{align*}
Hence, one cannot generically obtain $\divp_{\hyp}(\Prob\|P_\hyp\Prob)=0$, not even with the restriction
$\hyp = \hyp^{\star}_{\Prob}$. 
\end{remark}

For  $p=2$, $d$ the Euclidean distance on $\RR^{\sampleDim}$, and $\hyp_0=\hyp^{\star}_\Prob$ (which we recall is generally the primary interest case since
  our main bound~\eqref{eq:MainBoundExcessRisk} then gives a control of the excess risk with
  respect to the optimum), we still have
    $\divp_{\hyp^{\star}_\Prob}(\Prob\|P_{\hyp^{\star}_{\Prob}}\Prob)=0$ for certain tasks. In light of the above remark, this is a nontrivial property which is established for PCA in Lemma~\ref{lem:BiasPCA}, and  for $k$-means in the companion paper \citeparttwo. Beyond these specific situations, it is possible (under somewhat generic
  additional assumptions) to bound the two terms appearing in~\eqref{eq:biasbound} by (a power of) the risk itself, as established next.
    \begin{lemma}\label{lem:LemmaBiasTermBis}
Consider a compression-type task where $(\SampleSpace,d)$ is a \emph{separable} metric space. 
        Then 
  \begin{itemize}
\item Assume that $\SampleSpace$ has $d$-diameter bounded by $B$. Then for any $p> 1$, $h\in \HypClass$ and $\LossClass$-integrable distribution $\Prob$: %
  \begin{equation}
    \label{eq:rboundgeneralcasepgeq1}
    \divp_{\hyp}(\Prob\|P_\hyp\Prob) \leq 2p B^{p-1} \Risk(\Prob,\hyp)^{\frac{1}{p}}.
    \end{equation}
\item Assume that $\SketchingOperator: (\SampleSpace,d) \to (\RR^{\nMeasures},\norm{\cdot}_{2})$ (or $(\CC^{\nMeasures},\norm{\cdot}_{2})$) is $L$-Lipschitz. Then, for $\hyp \in \HypClass'$ and $p\geq 1$:
  \begin{equation}
    \label{eq:skopboundpgeq1}
    \norm{\SketchingOperatorProb(\Prob) - \SketchingOperatorProb(P_\hyp\Prob)}_2 \leq L \inf_{\hyp \in \HypClass} \Risk(\Prob,\hyp)^{\frac{1}{p}}.
  \end{equation}
  For $\hyp \in \HypClass$ and $p \leq 1$, if the space $\SampleSpace$ has $d$-diameter bounded by $B$:
  \begin{equation}
    \label{eq:skopboundpleq1}
    \norm{\SketchingOperatorProb(\Prob) - \SketchingOperatorProb(P_\hyp\Prob)}_2 \leq L B^{1-p} \inf_{\hyp \in \HypClass} \Risk(\Prob,\hyp).
  \end{equation}
\end{itemize}
\end{lemma}
The proofs of the above lemmas are in Appendix~\ref{sec:proofbiasterm}. For Lemma~\ref{lem:LemmaBiasTermBis}, optimal transport is exploited through connections between the considered norms and the norm $\norm{\Prob-\Prob'}_{\LipClass(L,d)} = L\cdot \norm{\Prob-\Prob'}_{\LipClass(1,d)}$, where $\LipClass(L,d)$ denotes the class of functions $f: (\SampleSpace,d) \to \RR$ that are $L$-Lipschitz. 
The two lemmas can be combined to express an ``explicit'' bound on $ \distIOPexgen_{\hyp^{\star}_{\Prob}}$, this is postponed to concrete examples.

\section{Illustration with Compressive PCA}\label{sec:CompressivePCA}

As a first simple illustration, this general compressive statistical framework can be applied to the example of PCA, where most of the tools already exist. Our aim is essentially illustrative, and focuses on controlling the excess risk, rather than to compare the results with state-of-the art PCA techniques. 

\paragraph{Definition of the learning task.}
The risk associated to the PCA learning problem is defined\footnote{for simplicity we assume centered distributions $\Exp_{\Sample\sim \Prob}\Sample=0$ and
  don't empirically recenter the data.} as $ \kPCArisk(\Prob,\hyp) = \Exp_{\Sample \sim \Prob}  \norm{\Sample-P_{\hyp} \Sample}_{2}^{2}$ with $P_{\hyp}$ the orthogonal projector onto subspace $\hyp$. 
It is minimized
by any subspace $\hyp^{\star}_{\Prob}$ associated with $\PCAdim$ largest eigenvalues of the matrix $\Cov_{\Prob} = \Exp_{\Sample \sim \Prob} \Sample\Sample^{T}$.  

It is well established \citep{FouRau13} that matrices that are approximately low-rank  can be estimated from partial linear observations under a certain Restricted Isometry Property (RIP). This leads to the following natural way to perform Compressive PCA.

\paragraph{Choice of a model set.} The ''natural'' model set from~\eqref{eq:DefLeastRestrictedModel} is $\ModelCT(\HypClass) = \{\Prob: \textrm{rank}(\Cov_{\Prob}) \leq k\}$. More generally we can consider as a model set $\Model_{r} := \{\Prob: \textrm{rank}(\Cov_{\Prob}) \leq r\}$, with $r \geq k$, so that $\Model_{r} \supset \ModelCT(\HypClass)$.

\paragraph{Choice of feature function.}  Choose (at random) a linear operator $\mathcal{M}: \RR^{\sampleDim \times \sampleDim} \to \RR^{\nMeasures}$ satisfying (with high probability) the RIP on low-rank matrices: for any $\mathbf{M} \in \RR^{\sampleDim \times \sampleDim}$ of rank at most $2r$, 
 \begin{equation}\label{eq:sym_RIP}
  1-\delta \leq \frac{\norm{\mathcal{M}(\mathbf{M})}_{2}^{2}}{\norm{\mathbf{M}}_{F}^{2}} \leq 1+\delta
  \end{equation}
with $\norm{\cdot}_{F}$ the Frobenius norm and $\delta < 1$. This is feasible with $m$ of the order of $r \sampleDim$, by taking the Frobenius inner product of $\mathbf{M}$ with $m$ independent random Gaussian matrices \citep[see e.g.][]{FouRau13}. 

Given these facts one can define the feature function as
 $\SketchingOperator: \SampleSpace = \RR^{\sampleDim} \rightarrow \RR^{\nMeasures}$ by
  \(
  \SketchingOperator(\sample) := \mathcal{M}(\sample \sample^{T}).
  \) 
  
  \paragraph{Sketch computation.}  Given sample points $\sample_1,\ldots,\sample_\nSamples$ in $\RR^\sampleDim$, compute the sketch $\vy$ as in~\eqref{eq:GenericSketching},
i.e., compute empirical estimates of random second moments of the distribution $\Prob$ of $\Sample$. 

\paragraph{\bf Learning from a sketch.}    Given a sketch vector $\vy$, estimate a solution of the optimization problem over positive semi-definite (p.s.d.) symmetric matrices ($\Cov \succcurlyeq 0$)
\begin{equation}\label{eq:PCAIdealDecoder}
  \hat{\Cov} := \arg\min_{\textrm{rank}(\Cov) \leq \rev{r}, \Cov \succcurlyeq 0} \norm{\mathcal{M}(\Cov)-\vy}^2_2.
\end{equation}
This step estimates the rank-$\rev{r}$ p.s.d. matrix whose sketch best matches that of the empirical matrix of second moments, in the least squares sense. Compute the eigen-decomposition $\hat{\Cov} = \mathbf{U}\mathbf{D}\mathbf{U}^{T}$ and output 
\begin{equation}\label{eq:PCAIdealDecoderBestHyp}
  \hat{\hyp} := \textrm{span}(\mathbf{U}(:,1:\PCAdim)).
\end{equation}
In Appendix~\ref{sec:PCA_theory} we control the excess risk of PCA
by relating the excess risk divergence $\divp_{\hyp_0}(\Prob\|\Prob')$ --- with $\hyp_{0} \in \HypClass$ an arbitrary hypothesis --- to the Frobenius
  norm $\norm{\Cov_\Prob-\Cov_{\Prob'}}_F$, and upper bounding the ``bias'' term~\eqref{eq:DefMDist2} appearing in
the generic bound of Theorem~\ref{thm:LRIPsuff_excess}, to obtain the following result:
\begin{theorem}\label{th:MainTheoremPCA}
Consider any probability distribution $\Prob$ with finite second moments  and any draw of $\sample_{i}$,
  $1 \leq i \leq \nSamples$ (represented by the empirical distribution $\empProb$). Applying the above approach yields, for any $s$, $1\leq s \leq r$:
\begin{equation} \label{eq:generalPCARiskBound}
  \kPCArisk(\Prob,\hat{\hyp}) - \kPCArisk(\Prob,\hyp^{\star}_{\Prob})
  \leq  c_\delta \sqrt{\frac{k}{s}} \sum_{j \geq r-s+2} \lambda_j(\Cov_{\Prob})+
  c'_\delta\sqrt{k}\norm{\mathcal{M}(\Cov_{\Prob}-\Cov_{\empProb})}_{2},
\end{equation}
where $\lambda_i(\Cov_{\Prob})$ are the eigenvalues of $\Cov_{\Prob}$ ranked in
decreasing order (with multiplicity), $c_\delta := 2\sqrt{2}\frac{\sqrt{1+\delta}}{\sqrt{1-\delta}} $, $c'_\delta :=2\sqrt{2}/\sqrt{1-\delta}$.
In particular:
\begin{equation}
  \label{eq:MainPCARiskBound}
  \kPCArisk(\Prob,\hat{\hyp}) - \kPCArisk(\Prob,\hyp^{\star}_{\Prob})
 \leq  c_\delta \sqrt{\frac{k}{r-k+1}}  \kPCArisk(\Prob,\hyp^{\star}_{\Prob})+  c'_\delta \sqrt{k}  
 \norm{\mathcal{M}(\Cov_{\Prob}-\Cov_{\empProb})}_{2}.
\end{equation}

\end{theorem}
\paragraph{Discussion:}
\begin{itemize}
\item {\bf Bias term.} The first term in the right hand side of~\eqref{eq:MainPCARiskBound} is a bias term that vanishes when the true risk is low. %
 Since it is proportional to the true risk, it leads to the (non-sharp) oracle inequality
$\kPCArisk(\Prob,\hat{\hyp}) \leq C_\delta(k,r) \kPCArisk(\Prob,\hyp^{\star}_{\Prob}) + c'_\delta\sqrt{k} \norm{\mathcal{M}(\Cov_{\Prob}-\Cov_{\empProb})}_{2}$. We  show in \citeparttwo, (using
Lemma~\ref{lem:LemmaBiasTermBis} and \eqref{eq:skopboundpleq1}) that this type of property also holds for Compressive $k$-medians; %
for Compressive $k$-means 
 we prove similar properties where the bias term is bounded by the square root of the true risk (using \eqref{eq:rboundgeneralcasepgeq1},\eqref{eq:skopboundpgeq1}). \\ It is notable that the bias multiplier $C_\delta(k,r)$ is of order $\sqrt{k}$
   if we use the natural model set ($r=k$), but drops
   to a constant independent of $k$ as soon as we choose e.g. the
   larger model set $\Model_r$ with $r=2k$.
    Thus, there appears to be a clear advantage, in the sense
   of the obtained bound, in choosing a reconstruction model that is
   larger than the natural model set $\ModelCT(\HypClass)$, while not significantly changing
   the magnitude of the number of required data sketches. At this point it is
   unclear to us if the inflation of the bias factor for the natural model is
   unavoidable or is just a technical artefact.
\item {\bf Sample complexity.} Regarding the second term, if we further assume that the support of $\Prob$ is contained in a Euclidean ball of radius $R$, then by the RIP~\eqref{eq:sym_RIP} we have a.s. 
$\norm{\mathcal{M}(\sample\sample^{T})}_{2} \leq \sqrt{1+\delta} \cdot R^{2}$
hence, by the vectorial Hoeffding's inequality \citep[see e.g.][]{Pinelis92}, we obtain with high probability w.r.t. data sampling that
$\sqrt{k} \norm{\mathcal{M}(\Cov_{\Prob})-\mathcal{M}(\Cov_{\empProb})}_{2}$ is of the order of
$R^2 \sqrt{k/n}$.
\item {\bf Root-$n$ consistency in a high-dimensional scenario}. 
    As noticed in the previous point, the statistical estimation error term in
    the sketched learning bound is of order $\sqrt{\PCAdim/\nSamples}$. Consider a high-dimensional situation where $\sampleDim$ is large and growing with $\nSamples$, and assume a polynomial spectral decay   $\lambda_j(\Cov_{\Prob}) \leq j^{-\alpha}$ with $\alpha>1$. Then by choosing $s=r/2$, the 
     bias term in~\eqref{eq:generalPCARiskBound} is of order $\sqrt{kr^{-(2\alpha-1)}}$.
     As a consequence, it is sufficient to take $r$ of order %
     $\min(d,n^{\frac{1}{2\alpha-1}})$
     so as to ensure that the bias term is at most of
     the same order as the statistical error term. This gives an advantage compared
     to the standard approach of
     storing all $d^2$ second order moments, as soon as $d> n^{\frac{1}{2\alpha-1}}$.
     For comparison, standard statistical analysis of PCA based on
     the uniform bound on the deviations of the empirical risk from its expectation
     (see e.g. \citealp{shawe2005}) leads
    to a control of order\footnote{More refined techniques \citep{blanchard2007,reiss2016} can lead to a convergence
       rate of the PCA excess risk of order $n^{-1}$ asymptotically, for $\PCAdim$ fixed, but depending on eigenvalue gaps. For the present discussion we compare
       ourselves to the simplest analysis available in the standard learning context.}
     $\sqrt{\PCAdim/\nSamples}$. Hence, using the sketched approach we can reduce
     storage/memory imprint significantly while keeping  statistical
     guarantees of the same order as in the standard setting.
   \item \rev{{\bf Relation to efficient kernel-PCA methods.} Methods in recent literature have been proposed
       to make kernel PCA more efficient, relying on Nyström subsampling \cite{Sterge:2020kpca} or
       on approximation of the kernel using random features \citep{Ullah:2018nips,Sriperumbudur:2020approx}. Even
       when restricting attention to a linear kernel, a direct comparison to our approach proves delicate. The Nyström subsampling
       method is very much taylored to the dual point of view, which is canonical in kernel methods
       (i.e. approximation of the kernel Gram matrix): this implicitly posits to store all training points in order to
       represent the output of the method as a kernel expansion (the computational gain concerns the storage and
       manipulation of the $(n,n)$ Gram matrix). On the other hand, the random feature approach could be construed as closer in spirit
       to ours, however rather than storing generalized moments as we do, it is in essence a (random) dimension reduction
       of the individual input points from kernel space to a finite-dimensional feature space where regular PCA is performed. Also the theoretical works
       of \cite{Ullah:2018nips,Sriperumbudur:2020approx} concern reconstruction in $L^2(\Prob)$ space and not in the original
       kernel space norm (lifting back the PCA projection found in approximate feature space into original kernel space
       proves a delicate question). In contrast, we consider sketching the data into generalized empirical moments
       but propose a reconstruction method directly in the relevant space with theoretical guarantees.
       It is to be noted however, that the
       results of \cite{Ullah:2018nips} posit a black-box PCA method applied in the finite (but high-)dimensional
       approximate feature space, and one could apply a sketching approach at this stage. This suggests the interesting proposal that efficient
       kernel PCA methods could be combined, rather than be in competition with, the sketching approach, though we
     did not push the idea further.}
   \end{itemize}

\paragraph{\bf Practical algorithms for learning and comparison to prior PCA-specific results.} One can consider several relaxations of the nonconvex optimization problem~\eqref{eq:PCAIdealDecoder} in order to perform compressive PCA. Beside convex relaxations using the minimization of the nuclear norm \citep[Section 4.6]{FouRau13}, \citet{Kabanava_2016} showed (in a complex-valued setting) that the rank constraint in~\eqref{eq:PCAIdealDecoder} can be relaxed when $\mathcal{M}$ is made of random rank-one projections, i.e. when $\SketchingOperator(\sample)= \frac{1}{\sqrt{m}} (|\inner{a_j,\sample}|^2)_{j=1,m}$ where $a_j \in \CC^{\sampleDim}$ are independent standard complex Gaussian vectors. In this setting, let
\begin{equation}
  \label{eq:PCAconvrelax}
  \hat{\Cov} := \arg\min_{\Cov \succcurlyeq 0} \norm{\mathcal{M}(\Cov)-\vy}^2_2,
\end{equation}
and the corresponding hypothesis $\hat{\hyp}$ obtained through~\eqref{eq:PCAIdealDecoderBestHyp}.
Combining \citep[Theorem~4 with $p=2$]{Kabanava_2016} with Equation~\eqref{eq:bdivfrob} in Section~\ref{sec:PCA_theory} and Equation~\eqref{eq:main-erm-ineq2} in Section~\ref{sec:ProofLRIPsuff_excess}, we have the following result: if $\nMeasures \geq C \PCAdim \sampleDim$ where $C$ is a universal constant, then with high probability on the draw of the $a_j$, for any $\sample_1,\ldots,\sample_n$, we have the control
\[
 \kPCArisk(\Prob,\hat{\hyp}) - \kPCArisk(\Prob,\hyp^{\star})
    \leq \norm{\hat{\Cov}-\Cov_{\Prob}}_F \leq  
    D_1  \kPCArisk(\Prob,\hyp^{\star}) +  D_2 \sqrt{k} \norm{\SketchingOperatorProb(\Prob)-\SketchingOperatorProb(\empProb)}_{2},
\]
where $D_1,D_2$ are positive universal constants that do not depend on $k$.

Hence, provided we use a model set of dimension $2r$ as discussed above, the error control~\eqref{eq:MainPCARiskBound} obtained
  via our general approach matches what can be obtained using directly the PCA-specific study of \citet{Kabanava_2016}.
  Two practical advantages of the latter are (a) that~\eqref{eq:PCAconvrelax} is a convex program,
  and (b) that the sketches are made using rank-one matrices, which are cheaper to store. 
Still, the guarantees obtained by our general approach is able to
match prior results for setting-specific methods. It will further permit the study of the less trivial setting of compressive clustering and compressive Gaussian mixture estimation as shown in the companion paper
\citeparttwo.

\section{Establishing the LRIP for random sketching operators}\label{sec:ChoiceSketch}

In this section, we investigate how to establish the LRIP \eqref{eq:lowerRIP_excess} (with $\eta=0$)
when the sketching operator $\SketchingOperatorProb$ is associated to random features. 
The approach uses connections with the notion of kernel mean embedding of probability distributions.

\subsection{Random features and kernel mean embeddings}\label{sec:Kernels}

\begin{definition}[Random feature map]\label{def:RFmap}
Consider $\FClass:=\{\rfeat\}_{\freq \in \freqSpace}$ a parameterized family of (real- or complex-valued) measurable functions, $\freqdist$  a probability distribution $\freqdist$ over the parameter set $\freqSpace$ (often $\freqSpace=\RR^\sampleDim$), and a sketch size $\nMeasures$.  A random feature map is defined by drawing $\nMeasures$ i.i.d parameters $(\freq_{j})_{j=1}^{\nMeasures}$ according to $\freqdist$ and setting
\begin{equation}\label{eq:DefRandomFeatureFunction}
\SketchingOperator(\sample) := \tfrac{1}{\sqrt{m}} \left(\rfeatj(\sample)\right)_{j=1,m}.
\end{equation}
\end{definition}
Any draw of the feature function $\SketchingOperator$ defines a positive semi-definite kernel between samples $\kappa_{\SketchingOperator}(\sample,\sample') := \langle \SketchingOperator(\sample),\SketchingOperator(\sample')\rangle_{\RR^{\nMeasures}}$ (or $\langle \SketchingOperator(\sample),\SketchingOperator(\sample')\rangle_{\CC^{\nMeasures}}$). Compressive learning is deeply connected to kernel mean embeddings of probability distributions, as the related sketching operator $\SketchingOperatorProb$  defines a so-called \emph{kernel mean embedding} between probability distributions which are $\FClass$-integrable.

\begin{definition}[Kernel mean embedding, Mean Map Discrepancy \citep{Gretton2007, Sriperumbudur2010}]\label{def:MMD}
Any positive semi-definite kernel $\kappa(\cdot,\cdot)$ in the sample space is associated to a Mean Map Embedding (a kernel between distributions). By abuse of notation, we keep the notation $\kernel$ for both the expression of the kernel in the sample space and of the corresponding kernel for probability distributions with appropriate integrability
\begin{equation}
\label{eq:DefMeanMapEmbedding}
\kappa(\Prob,\Prob') := \Exp_{\Sample \sim \Prob} \Exp_{\Sample' \sim \Prob'} \kernel(\Sample,\Sample').
\end{equation}
The associated Maximum Mean Discrepancy (MMD) metric is
\begin{equation}\label{eq:DefMMD}
\norm{\Prob-\Prob'}_{\kappa} := \sqrt{\kappa(\Prob,\Prob)-2\operatorname{Re}(\kappa(\Prob,\Prob'))+\kappa(\Prob',\Prob')}.
\end{equation}
\end{definition}

The average kernel $\kernel$ associated to $(\FClass,\freqdist)$, will play a key role in establishing the LRIP.
Given  $\sample,\sample' \in \SampleSpace$,  the expectation of $\kappa_{\SketchingOperator}(\sample,\sample') = \frac{1}{m} \sum_{j=1}^{m} \rfeat_{j}(\sample)\overline{\rfeat_{j}}(\sample')$ over the draws of $\freq_{j}$ is
\begin{equation}\label{eq:DefIntegralRepresentation}
  \kernel(\sample,\sample')=\Exp_{\freq\sim\freqdist} \kernel_\SketchingOperator(\sample,\sample')=
 \Exp_{\freq\sim\freqdist}\rfeat(\sample)\overline{\rfeat(\sample')}.
\end{equation}
Similarly, given $\Prob,\Prob'$, the squared MMD $\normkern{\Prob-\Prob'}^{2}$ with this kernel is the expectation of 
\[
\norm{\Prob-\Prob'}_{\kappa_{\SketchingOperator}}^{2} =  
\norm{\SketchingOperatorProb(\Prob)-\SketchingOperatorProb(\Prob')}_2^{2} = \frac{1}{\nMeasures} \sum_{j=1}^{\nMeasures} \abs{\Exp_{\Sample \sim \Prob} \rfeatj(\Sample)-
\Exp_{\Sample' \sim \Prob'} \rfeatj(\Sample')}^{2}.
\]
A characterization of the MMD that we will leverage throughout this section is that for any $\Prob,\Prob'$,
\[
\normkern{\Prob-\Prob'}^{2} = \Exp_{\freq \sim \freqdist} \abs{\Exp_{\Sample \sim \Prob} \rfeat(\Sample)-
\Exp_{\Sample' \sim \Prob'} \rfeat(\Sample')}^{2}.
\]

We observe that  $\SketchingOperatorProb$ satisfies the LRIP~\eqref{eq:lowerRIP_excess} (with $\eta=0$) for a given model set $\Model$ if, and only if,  the metric %
$\dnormloss{\Prob-\Prob'}{}$
is dominated by $\norm{\Prob-\Prob'}_{\kernel_{\SketchingOperator}}$ for $\Prob,\Prob' \in \Model$.
Our  overall strategy to check that a random feature function $\SketchingOperator$ defined by $\FClass$ and $\freqdist$ satisfies the LRIP~\eqref{eq:lowerRIP_excess} (with $\eta=0$) with controlled sketch dimension $\nMeasures$ will be to:
\begin{enumerate}
\item prove that the average kernel $\kernel$ defined by~\eqref{eq:DefIntegralRepresentation} satisfies the {\em Kernel LRIP}
\begin{equation}
\label{eq:KernelLRIP}
\normdloss{\mProb-\mProb'}{} \leq C_{\kernel} \norm{\mProb-\mProb'}_{\kappa},\quad \forall \mProb,\mProb' \in \Model;
\end{equation}
\item in the spirit of compressive sensing theory, use concentration of measure and covering arguments to show that for any $0<\delta<1$, for large enough $\nMeasures$, with high probability on the draw of $\freq_{j}$, 
\begin{equation}\label{eq:KernelRIPL2}
1-\delta 
\leq \frac{\norm{\SketchingOperatorProb(\mProb)-\SketchingOperatorProb(\mProb')}_2^{2}}{\norm{\mProb-\mProb'}_{\kappa}^{2}}
= \frac{\norm{\mProb-\mProb'}_{\kernel_{\SketchingOperator}}^{2}}{\norm{\mProb-\mProb'}_{\kappa}^{2}} \leq 1+\delta,\quad \forall \mProb,\mProb' \in \Model
\end{equation}
so that the kernel LRIP~\eqref{eq:KernelLRIP} actually holds with $\kappa_{\SketchingOperator}$ instead of $\kappa$ and constant $C_{\kernel_{\SketchingOperator}} := C_{\kernel}/\sqrt{1-\delta}$.
\end{enumerate}
\begin{remark}
The expression~\eqref{eq:KernelRIPL2} expresses the control of the \emph{relative error} of approximation \emph{of the MMD}, \emph{restricted} to certain distributions. This contrasts with state of the art results on random features 
\rev{that either control uniformly the approximation of the kernel $\abs{\kernel_{\SketchingOperator}(\cdot,\cdot)-\kernel(\cdot,\cdot)}$ \citep[see e.g.][]{Sriperumbudur:2015to} or the approximation of functions in the RKHS induced by $\kernel$ by those in the RKHS induced by $\kernel_{\SketchingOperator}$ \citep{Bach:2015ux}.} These types of controls are indeed of a different nature compared to ours, and none seems to be a direct consequence of the others.
\end{remark}

\subsection{Ingredients to verify the Lower Restricted Isometry Property}

In sight of the inequalities~\eqref{eq:KernelLRIP},\eqref{eq:KernelRIPL2} we need to prove, the analysis will focus on the so-called normalized secant set of the model $\Model$ with respect to the average kernel $\kernel$, defined as follows \citep[see, e.g. ][]{Dirksen:2014wl,puy:hal-01203614}: 
\begin{definition}[Normalized secant set] 
  The \emph{normalized secant set} of a model set $\Model$ with respect to a kernel $\kernel$ is the following subset of the set of finite, signed measures  (see Appendix~\ref{sec:FiniteSignedMeasures}) with appropriate integrability
\begin{equation}\label{eq:DefNormalizedSecantSet}
\secant_{\kernel} = \secant_\kernel(\Model) := \set{\frac{\mProb-\mProb'}{\normkern{\mProb-\mProb'}}: \mProb,\mProb'\in\Model, \normkern{\mProb-\mProb'}>0}.
\end{equation} 
\end{definition}
Using the secant set, the LRIP \eqref{eq:KernelRIPL2} is equivalent to
\begin{equation}\label{eq:KernelLRIP2reformulated}
\abs{ \norm{\SketchingOperatorProb(\HH)}^2_2 -1} \leq \delta,\quad \forall \mu \in \secant_\kernel
\end{equation}
The radius of $\secant_\kernel$ with respect to certain function norms will play an important role. Since this notion will come up repeatedly in the analysis, we introduce the following notation, which will be heavily used in the sequel.
Given a norm $\norm{\cdot}$ on measures, %
the radius of a subset $\mathcal{E}$ of finite signed measures is denoted
\begin{equation}\label{eq:DefSetRadius}
\norm{\mathcal{E}} := \sup_{\HH \in \mathcal{E}} \norm{\HH}.
\end{equation}
With these definitions we can observe that
\[
\sup_{\mProb,\mProb' \in \Model, \normkern{\mProb-\mProb'} > 0}  \frac{\normfclass{\mProb-\mProb'}{G}}{\normkern{\mProb-\mProb'}} = \normfclass{\secant_{\kernel}(\Model)}{G};
\]
where we recall that the metric $\normfclass{\cdot}{G}$ is defined in~\eqref{eq:DefFNorm}. In particular, the constant from~\eqref{eq:KernelLRIP} can be equivalently rewritten as $C_{\kernel} := \dnormloss{\secant_{\kernel}}{}$.

Concerning~\ref{eq:KernelRIPL2}, the strategy to establish it will rely on the following two
  quantities:
first, a \emph{concentration function} $\ConcFn(t)$ characterizing the pointwise (i.e. for fixed $\mProb,\mProb'\in \Model$) concentration of $\norm{\SketchingOperatorProb(\mProb)-\SketchingOperatorProb(\mProb')}_{2}^{2}$ around its expectation; secondly, certain covering numbers of $\secant_\kernel$ needed to step from pointwise to uniform concentration.

Classical arguments from compressive sensing \citep{Richard-Baraniuk:2008aa,Eftekhari:2013ti,puy:hal-01203614,Dirksen:2014wl,FouRau13} prove that certain random linear operators satisfy the RIP by relying on pointwise concentration inequalities.
Similarly, a first step to establish that the inequalities~\eqref{eq:KernelRIPL2} hold with high probability consists in assuming first a pointwise version of the same, i.e., for any choice of $\nMeasures$ in~\eqref{eq:DefRandomFeatureFunction}:
\begin{equation}\label{eq:PointwiseConcentrationFn}
  \text{for any } \HH \in \secant_\kernel : \qquad \mathbb{P}\left(
  \abs{ \norm{\SketchingOperatorProb(\HH)}^2_2 -1}
  \geq t\right)
  \leq 2 \exp\left(-\frac{\nMeasures}{\ConcFn(t)}  \right),
\end{equation}
for some \emph{concentration function} $t \mapsto \ConcFn(t)$ that should ideally be as small as possible.  The following result shows that the radius $\normfclass{\secant_{\kernel}}{F}$ can be used to control such a concentration function.
\begin{lemma}\label{le:PointwiseConcentrationLemma}
Consider a family of functions $\FClass := \{\rfeat\}_{\freq \in \Omega}$, $\nMeasures$ parameters $(\freq_j)_{j=1}^{m}$ drawn i.i.d. according to some distribution $\Lambda$ on $\Omega$, and $\SketchingOperatorProb$ the (random) operator induced (see~\eqref{eq:SketchingOperatorProbDef}) by the feature function
\(
\SketchingOperator(\sample) := \tfrac{1}{\sqrt{m}} \left(\rfeatj(\sample)\right)_{j=1}^{m}.
\)
Denoting $\kernel$ the associated average kernel (cf~\eqref{eq:DefIntegralRepresentation}) we have $\normfclass{\secant_{\kernel}}{F} \geq 1$. Moreover, if $\normfclass{\secant_{\kernel}}{F} < \infty$ then~\eqref{eq:PointwiseConcentrationFn} holds for all %
$\HH  \in \secant_\kernel$,
with
\begin{equation}\label{eq:ConcFnFromConcCst}
\ConcFn(t) \leq %
2t^{-2}(1+t/3) \cdot \normfclass{\secant_{\kernel}}{F}^{2},\qquad \forall t>0.
\end{equation}
\end{lemma}
The proof is in Appendix~\ref{sec:ProofThmMainLRIP}.
Observe that the above estimate only depends on the choice of the feature family $\FClass$, and holds for any feature sampling distribution $\freqdist$. More refined %
estimates for mixture models, exploiting
moments of $\freqdist$ rather than a uniform bound, are provided in the companion paper \citeparttwo\ and used to obtain concrete estimates for Compressive Clustering and Compressive GMM.
Finally, we will extrapolate pointwise concentration~\eqref{eq:PointwiseConcentrationFn} to all pairs $\mProb,\mProb' \in \Model$ using covering numbers of the normalized secant set %
with respect to an appropriate metric.
 
\begin{definition}[Covering number] \label{def:covnum}
  The \emph{covering number} $\covnum{d(\cdot,\cdot)}{S}{\coveps}$ of a set $S$ with respect to a (pseudo)metric\footnote{Further reminders on metrics, pseudometrics, and covering numbers are given in Appendix~\ref{sec:notations_definitions}.} $d(\cdot,\cdot)$ is the minimum number of closed balls of radius $\coveps$ with respect to $d(\cdot,\cdot)$ with centers in $S$ needed to cover $S$.
\end{definition}
As the normalized secant set is a subset of the infinite-dimensional space of finite-signed measures, it is not obvious when its covering numbers are finite. Controlling them can be nontrivial, yet this is feasible on a case by case basis as will be illustrated in the companion paper \citeparttwo.

Covering numbers and pointwise concentration can be then combined to give rise to the following result
(whose proof is in Appendix~\ref{sec:ProofThmMainLRIP}) where the logarithm of the covering numbers somehow captures an intrinsic dimension of the considered learning task:
\begin{theorem}\label{thm:mainLRIP}
  Consider $\FClass := \set{\rfeat}_{\freq \in \Omega}$ a family of functions, $\freqdist$ a probability distribution on $\Omega$, $\SketchingOperator$ the associated random feature function and $\kernel$ the corresponding average kernel. Consider the pseudometric on $\FClass$-integrable probability distributions\footnote{In fact, we consider the extension of $d_\SketchingOperator$ to finite,
    $\FClass$-integrable signed measures, see Appendix~\ref{sec:FiniteSignedMeasures}.} 
\begin{equation}\label{eq:DefYetAnotherMetric}
d_\FClass(\Prob,\Prob'):=\sup_{\freq\in \freqSpace}\abs{\abs{\Exp_{\Sample \sim \Prob}\rfeat(\Sample)}^2 - \abs{\Exp_{\Sample' \sim \Prob'}\rfeat(\Sample')}^2}.
\end{equation} 
Consider a model set $\Model$ and $\secant_{\kernel}=\secant_{\kernel}(\Model)$ its normalized secant set. Assume that $\secant_\kernel$ has finite covering numbers with respect to the pseudometric $d_\FClass$. For $0<\delta,\zeta<1$, if 
\begin{equation}
\nMeasures \geq \ConcFn(\delta/2) \cdot
\log\Big(2\covnum{d_\FClass}{\secant_\kernel}{\delta/2}/\probLevel\Big),
\end{equation}
then, with probability at least $1-\probLevel$ on the draw of $(\freq_{j})_{j=1}^{m}\stackrel{i.i.d.}{\sim}\Lambda$, the operator $\SketchingOperatorProb$ induced by $\SketchingOperator$ (cf \eqref{eq:DefRandomFeatureFunction} and \eqref{eq:SketchingOperatorProbDef}) satisfies 
\begin{equation}\label{eq:mainLRIPPureKernel}
1-\delta \leq 
\frac{\norm{\SketchingOperatorProb(\mProb)-\SketchingOperatorProb(\mProb')}_{2}^{2}}{\norm{\mProb-\mProb'}_{\kappa}^{2}}
\leq 1+\delta, \qquad \forall \mProb,\mProb' \in \Model.
\end{equation}
When~\eqref{eq:mainLRIPPureKernel} holds, the LRIP~\eqref{eq:lowerRIP_excess} with $\eta=0$ holds with constant 
$C_\SketchingOperatorProb :=\frac{\dnormloss{\secant_{\kernel}}{}}{\sqrt{1-\delta}}$ and $\eta = 0$.
\end{theorem}

\subsection{Summary and applications}
\begin{table}[ht!]
{\small
 \renewcommand{\arraystretch}{1.5}
\begin{tabular}{|p{0.15\textwidth}|p{0.15\textwidth}|p{0.26\textwidth}|p{0.35\textwidth}|}
\hline
 Task & 
 PCA & 
 $k$-med./means ($p=1$/$p=2$)  & 
 Gaussian Mixture Model.   \\
 \hline
\hline
 Hypothesis $\hyp$  & subspace & $k$ cluster centers &  mixture $\Prob_{\hyp}$ of $k$ Gaussians\\
& $h \subset \RR^{\sampleDim}$ & $c_1,\ldots,c_k \in \RR^{\sampleDim}$ & -means $c_{l} \in \RR^{\sampleDim}$ \\
& $\dim \hyp = k$ & & -covar. $\Cov_{l} \in \RR^{\sampleDim \times \sampleDim}$\\
& & & -mixture parameters $\alpha_{l}$\\
 \hline
 Loss  $\loss(\sample,\hyp) $  & $\norm{\sample-P_{\hyp} \sample}_{2}^{2}$& $\min_{1 \leq l \leq k} \norm{\sample-c_{l}}_{2}^{p}$& $-\log \Prob_\hyp(\sample)$\\
 \hline
 \hline
 Feature function   & quadratic polyn.  & weighted Fourier features  & Fourier features\\ 
 $\SketchingOperator(\sample)$ & $\tfrac{1}{\sqrt{m}}\left(\sample^{T}\mL_{j}\sample\right)_{j=1}^{\nMeasures}$ & $\tfrac{1}{\sqrt{m}}\left(\frac{e^{\jmath\freq_{j}^{T}\sample}}{w(\freq_{j})} \right)_{j=1}^{\nMeasures}$ & $\tfrac{1}{\sqrt{m}}\left(e^{\jmath\freq_{j}^{T}\sample}\right)_{j=1}^{\nMeasures}$\\
\hline
 Sampling law $\freqdist$ & $\mathbb{P}(\mL) \propto e^{-\norm{\mL}_{F}^{2}}$ & 
 $\mathbb{P}(\freq) \propto w^{2}(\freq) e^{ - \frac{s^{2}\norm{\freq}^2}{2}}$  & $\mathbb{P}(\freq) \propto e^{ - \frac{s^{2}\normmah{\freq}{\covar^{-1}}^2}{2}}$\\
 \hline
 Average kernel $\kernel(\sample,\sample')$ & 
 $\norm{\sample\sample^{T}-\sample'\sample'^{T}}_{F}^{2}$ &
 $\exp\left(-\tfrac{\norm{\sample-\sample'}_{2}^{2}}{2s^{2}}\right)$ &
 $\exp\left(-\tfrac{\normmah{\sample-\sample'}{\covar}^{2}}{2s^{2}}\right)$
 \\
 \hline
 \hline
Proxy  &  & & \\
$\proxyRisk(\hyp,\vy)$ & 
Low-rank\qquad recovery & $\displaystyle\min_{\alpha \in \Simplex_{\PCAdim-1}}
\|\sum_{l=1}^{k}\alpha_{l}\SketchingOperator(c_{l})-\vy\|_{2}$ & $\norm{\SketchingOperatorProb(\Prob_{\hyp})-\vy}_{2}$ \\
from ~\eqref{eq:GenericRiskProxy} & & & \\
 \hline
 Restrictions on & N/A & $\min_{c_l \neq c_{l'}}\norm{c_{l}-c_{l'}}_{2} \geq 2\sep$ & $\min_{c_l \neq c_{l'}}\normmah{c_{l}-c_{l'}}{\covar} \geq 2\sep$\\
hypothesis class  & & $\max_{l} \norm{c_{l}}_{2} \leq R$ & $\max_{l} \normmah{c_{l}}{\covar} \leq R$\\
$\HypClass$ when optim- & & $\sep := 4s\sqrt{\log(ek)}$ & $\sep := 4\sqrt{(2+s^{2})\log(ek)}$ \\
\-mizing proxy & & $w(\freq) := 1+\frac{s^{2}\norm{\freq}_{2}^{2}}{\sampleDim}$ & known covariance $\Cov_{l} = \Cov, \forall l$\\
\hline
Sketch size $\nMeasures$ %
 & 
 $\PCAdim d$ &  
 $\PCAdim^{2} \sampleDim \log(e\PCAdim\sampleDim  R/\sep)\log^{2} (ek)$ & 
$k^{2}\sampleDim \log(ekdR)\log^{2}(ek)$\\
      &&& when $s=\sqrt{d}$, cf  \citeparttwo (Table~1)
          for other values of $s$.\\
 \hline
\end{tabular}
\caption{%
Summary of the application of the framework on our three main examples (detailed in Section~\ref{sec:CompressivePCA} and the companion paper \citeparttwo)
in $\SampleSpace = \RR^{\sampleDim}$. $\Simplex_{\PCAdim-1}$ denotes the $(\PCAdim-1)$-dimensional simplex (i.e. the sphere with respect to the $\ell^1$-norm in the non-negative orthant of $\RR^{\PCAdim}$), and $\normmah{\sample}{\covar} = \sample^{T}\covar^{-1}\sample$ the Mahalanobis norm associated to the positive definite covariance matrix $\covar$. 
The order of the sketch size is indicated up to universal  numerical multiplicative factor and logarithmic dependencies on the parameters $\coveps$ and $\zeta$ from Theorem~\ref{thm:mainLRIP}. The displayed average kernels are up to a multiplicative constant. 
\label{tab:summary}}}
\end{table}

To briefly summarize the results in this section, in order to establish the LRIP property
with respect to a given model $\Model$ in the context
of a sketching operator $\SketchingOperator$ associated to a family of random features $\FClass$ and
feature sampling distribution $\Lambda$ we proceed as follows. After identifying the associated average kernel
\eqref{eq:DefMeanMapEmbedding}, the key quantities to estimate relative to the normalized
secant $\secant_\kernel(\Model)$ are its radius $\dnormloss{\secant_{\kernel}}{}$
(which serves as
a measure of compatibility between the kernel, the learning task, and the model set $\rev{\Model}$),
the pointwise concentration function $\ConcFn(.)$ from~\eqref{eq:PointwiseConcentrationFn}, and the covering numbers of $\secant_\kernel$
with respect to the distance $d_\FClass$ from~\eqref{eq:DefYetAnotherMetric}.

Even though the above ingredients and results may look quite abstract at this stage, we can turn them into concrete estimates on several examples. %
The resulting guarantees are summarized in  Table~\ref{tab:summary} for the examples developed in detail in the companion paper \citeparttwo
(compressive $k$-Means, $k$-medians and GMM).

The random sketching results developed in the present section can
also be used to revisit the illustrative PCA example from Section~\ref{sec:CompressivePCA}.
Namely, while we have directly lifted from existing literature the RIP property~\eqref{eq:sym_RIP} 
for random Gaussian sketching matrices applied to low-rank covariance matrices, the arguments used there to establish
this property follow in essence the canvas of this section (pointwise concentration of the random
operator to its average, then unifom concentration via appropriate covering number arguments).
In this context, the squared MMD with respect to the averaged kernel is precisely the Frobenius norm
between covariance matrices. The additional ingredient needed to
complete the analysis is to relate the PCA excess risk to
the Frobenius norm of differences of low rank matrices %
(see~\eqref{eq:bdivfrob} in the technical Appendix~\ref{sec:PCA_theory}, which can be reinterpreted as a bound on $\dnormloss{\secant_{\kernel}}{}$ in the PCA setting).

\section{Conclusion and perspectives}\label{sec:future}
The principle of compressive statistical learning is to learn from large-scale collections by first summarizing the collection into a sketch vector made of empirical (random) moments, before solving a nonlinear least squares problem. The main contribution of this paper is to set up a general mathematical framework for compressive statistical learning and to demonstrate on 
an example (compressive PCA)
that the excess risk of this procedure can be controlled, as well as the sketch size. 
The companion paper \citeparttwo\ completes the illustration of the framework by considering two more examples: compressive clustering and compressive Gaussian mixture estimation --- with fixed known covariance.

\paragraph{Sharpened estimates?} Our demonstration of the validity of the compressive statistical learning framework for certain tasks is, in a sense, qualitative, and we expect that many bounds and constants are sub-optimal. 
A number of non-sharp oracle inequalities have been established in the course of our endeavor. 
A particular question is to obtain more explicit and/or tighter control of the bias term $ \distIOPexgen_{\hyp^{\star}_{\Prob}}(\Prob,\ModelCT(\HypClass))$, %
and to understand whether Lemma~\ref{lem:LemmaBiasTerm}, which relates this bias term to the optimal risk, can be tightened and/or extended to other loss functions.
In the same vein, as fast convergence rates for the excess risk can be established for certain classical statistical learning tasks under appropriate conditions (see e.g. \citep{Levrard:2013ho} for the case of $k$-means), it is natural to wonder whether the same holds for compressive statistical learning.

\paragraph{Links with neural networks.}
From an algorithmic perspective, the sketching techniques we have explicitly characterized in this paper have a particular structure which is reminiscent of a one-layer (random) neural network with subsequent averaging over multiple data points.
Indeed, when the sketching function $\SketchingOperator$ corresponds to random Fourier features, its computation for a given vector $\sample$ involves first multiplication by the matrix $\mathbf{W} \in \RR^{\nMeasures \times \sampleDim}$ whose rows are the selected frequencies $\freq_{j} \in \RR^{\sampleDim}$, then pointwise application of the $e^{\jmath \cdot}$ nonlinearity. 
Here we consider random Fourier \emph{moments}, hence a subsequent \emph{averaging} operation is performed. As we have seen, this draws a link with the MMD, as is done e.g. in the so-called MMD-GANS \citep{Li2015, Binkowski2018}, where the so-called discriminator is a neural net trained to compute MMDs over batches of samples. 

This suggests that our analysis could help analyze the tradeoffs between reduction of the information flow (dimension reduction) across multiple layers of such networks and the preservation of statistical information \citep{Shwartz-Ziv2017}. For example, this could explain why the pooled output of a layer is rich enough to cluster the input patches. Given the focus on drastic dimension reduction, this seems very complementary to the work on the invertibility of deep networks and pooling representations with random Gaussian weights \citep{JoanBruna:2014vc,Giryes:2015vo,Gilbert:2017wy}. Finally, we mention the recent popularity of networks with random weights in statistical physics \citep{Gabrie2018} and in analyzing the initialization point of optimization algorithms with a kernel characterization \citep{Jacot2018, Bietti2019}, for which information-preservation (non-degeneracy during training) is also an essential feature.

\paragraph{Privacy-aware learning via sketching?}
The reader may have noticed that, while we have defined sketching in~\eqref{eq:GenericSketching} 
as the empirical average of (random) features $\SketchingOperator(\sample_{i})$ over the training collection (or in fact the training \emph{stream}), the essential feature of the sketching procedure is to provide a good empirical estimator of the sketch vector $\SketchingOperatorProb(\Prob) = \Exp_{\Sample \sim \Prob} \SketchingOperator(\Sample)$ of the underlying probability distribution. 
A consequence is that one can envision \emph{other sketching mechanisms}, in particular ones more compatible with privacy-preservation constraints \citep{Duchi:2012ux}. For example, one could average  $\SketchingOperator(\sample_i+\xi_{i})$, or $\SketchingOperator(\sample_i)+\xi_{i}$, or $\mathbf{D}_{i} \SketchingOperator(\sample_{i})$, etc., where $\xi_{i}$ is a heavy-tailed random vector drawn independently from $\sample_{i}$, and $\mathbf{D}_{i}$ is a diagonal ``masking'' matrix with random Bernoulli $\set{0,1}$ entries. An interesting perspective is to characterize such schemes in terms of tradeoffs between differential privacy  and ability to learn from the resulting sketch. Preliminary results in this direction have been recently achieved \cite{schellekens:hal-02060208,chatalic:hal-02496896}.

\paragraph{Recipes to design sketches for other learning tasks through kernel design?}
\begin{figure}[ht!]
\centering
\includegraphics[width=0.8\columnwidth]{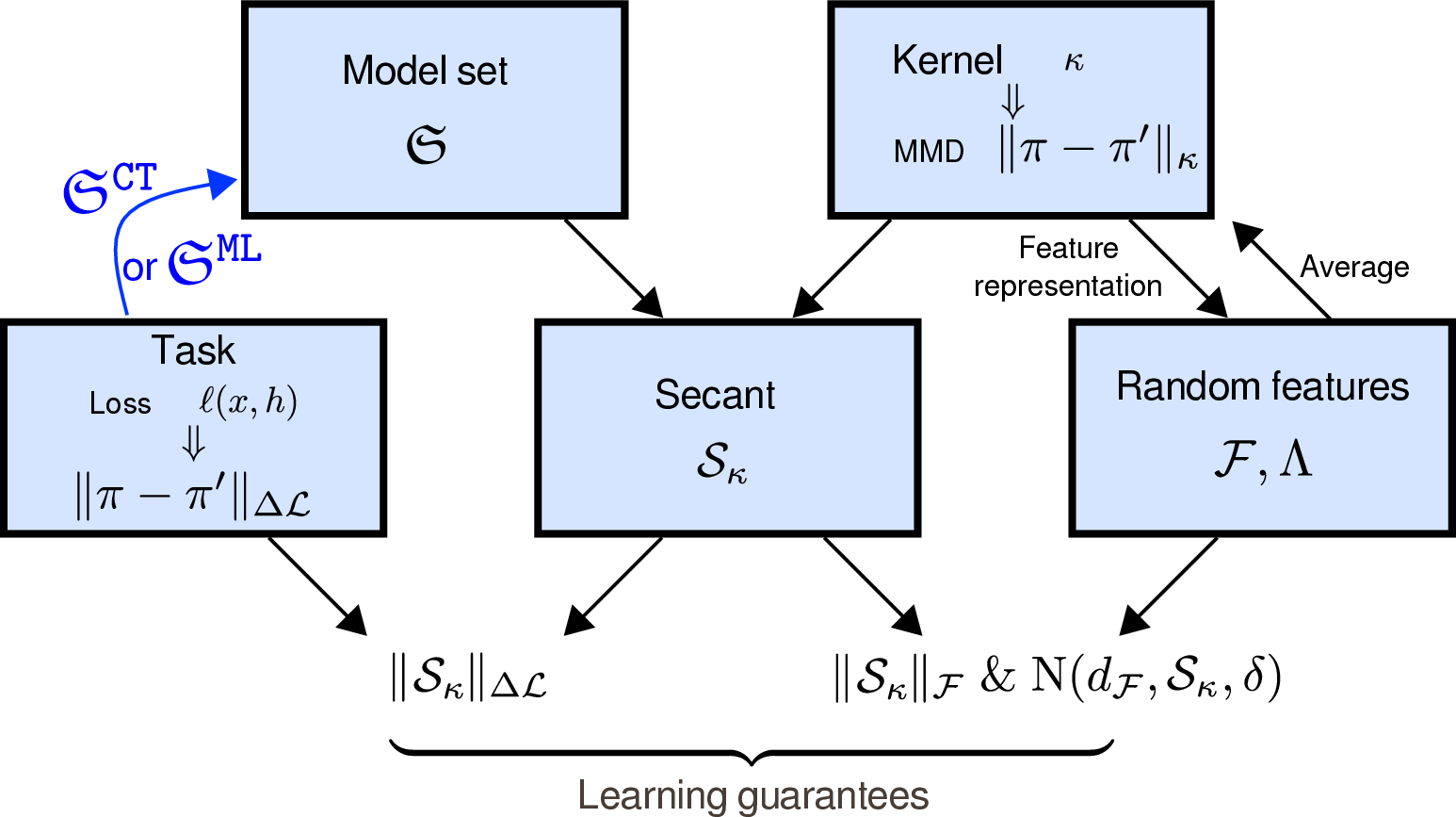}
\caption{A representation of the links between different concepts in this paper.}\label{fig:summary}
\end{figure}
Given the apparent genericity of the proposed compressive statistical learning framework, a particular challenge is to extend it beyond the learning tasks considered in this paper and its companion \citeparttwo. Kernel versions of these tasks (\emph{kernel} PCA, \emph{kernel} $k$-means, or \emph{spectral} clustering) appear as the most likely immediate extensions. They are expected to lead to sketching architectures reminiscent of \emph{two-layer} convolutional neural networks with additive pooling. 
Compressive supervised classification and compressive regression are also natural candidate tasks
in a learning setting, but seem more challenging.

Given a learning task, the main bottleneck is to find an adequate sketching function $\SketchingOperator(\cdot)$. 
As illustrated on Figure~\ref{fig:summary}, this primarily relies on the quest for a \emph{task-compatible kernel}, i.e., one satisfying the Kernel LRIP~\eqref{eq:KernelLRIP}. Subsequent technical steps would rely on the identification of an integral representation of this kernel using random features with the right concentration properties, and establishing that the associated secant set has finite covering dimension with respect to the feature-based metric~\eqref{eq:DefYetAnotherMetric}. On a case by case basis, one may have to identify the analog of the separation conditions apparently needed for compressive $k$-means, see \citeparttwo.

Vice-versa, one could wonder which family of learning tasks is compatible with a given kernel. In other words, how ``universal'' is a kernel, and how much can be learned from a single sketched representation of a database ? We expect that tasks such as compressive ranking, which involve pairs, triples, etc. of training samples, may require further extensions of the compressive statistical learning framework, to design sketches based on $U$-statistics rather than plain moments. These would lead to sketches linear in the product probability $\Prob \otimes \Prob$ instead of $\Prob$. The investigation of such extended scenarios is expected to benefit from analogies with the lifting techniques used in phaseless reconstruction, see e.g. \citep{Candes2013}.

\section*{Acknowledgements}

This work was supported in part by the European Research Council, PLEASE project (ERC-StG-2011-277906), the german DFG (FOR-1735 ``Structural inference in statistics'', SFB-1294 ``Data Assimilation''), the Franco-German University through the binational Doktorandenkolleg CDFA 01-18,  and the ANR (ANR-19-CHIA-0021-01, project BISCOTTE; ANR-19-CHIA-0009, project AllegroAssai).
R{\'e}mi Gribonval is very grateful to Michael E. Davies for many enlightening discussions around the idea of compressive statistical learning since this project started several years ago. The authors also wish to warmly thank Bernard Delyon and Adrien Saumard, as well as Gabriel Peyr{\'e} and Lorenzo Rosasco for their constructive feedback on early versions of this manuscript.

\pagebreak
\part*{Appendix}
\appendix
\section{Notations, definitions}\label{sec:notations_definitions}

In this section we group all notations and some useful classical results.

\subsection{Metrics and covering numbers}\label{sec:metricdefin}
\begin{definition} A \textbf{pseudometric} $d$ over a set $X$ satisfies all the axioms of a metric, except that $d(x,y)=0$ does not necessarily imply $x=y$. Similarly, a \textbf{semi-norm} $\norm{\cdot}$ over a vector space $X$ satisfies the axioms of a norm except that $\norm{x}=0$ does not necessarily imply $x=0$.
\end{definition}

\begin{definition}[Ball, $\coveps$-covering, Covering number]\label{def:covering}
Let $(X,d)$ be a pseudometric space. For any $\coveps >0$ and $x\in X$, we denote $\Ball_{X,d}(x,\coveps)$ the \textbf{closed ball} of radius $\coveps$ centered at the point $x$:
\[
\Ball_{X,d}(x,\coveps)=\left\lbrace y\in X,~d(x,y)\leq\coveps\right\rbrace.
\]
Let $Y \subseteq X$ be a subset of $X$. A subset $Z\subseteq Y$ is a \textbf{$\coveps$-covering} of $Y$ if $Y\subseteq\bigcup_{z\in Z} \Ball_{X,d}(z,\coveps)$. The \textbf{covering number} $\covnum{d}{Y}{\coveps} \in \mathbb{N}\cup\lbrace +\infty \rbrace$ is the smallest $k$ such that there exists a $\coveps$-covering of $Y$ made of $k$ elements $z_i \in Y$.
\end{definition}

\subsection{Finite signed measures} \label{sec:FiniteSignedMeasures}
The space $\FSMSpace$ of finite signed measures on the measurable sample space $(\SampleSpace,\mathfrak{Z})$ is a linear space that contains the set of probability distributions on $(\SampleSpace,\mathfrak{Z})$. By the Hahn-Jordan theorem, any finite signed measure $\HH \in \FSMSpace$ can be decomposed into a positive and a negative part, $\HH = \HH_{+}-\HH_{-}$, where both $\HH_{+}$ and $\HH_{-}$ are non-negative finite measures on $(\SampleSpace,\mathfrak{Z})$, hence $\HH_{+} = \alpha \Prob_{+}$ and $\HH_{-} = \beta \Prob_{-}$ for some probability distributions $\Prob_{+},\Prob_{-}$, and non-negative scalars $\alpha,\beta \geq 0$. A real-valued, measurable function $f$ on $(\SampleSpace,\mathfrak{Z})$ is said integrable with respect to $\HH$ when it is integrable both with respect to $\HH_+$ and $\HH_-$. Noticing that the expectation of an integrable function $f$ is linear in the considered probability distribution, we adopt the inner product notation for expectations:
\[
\inner{\Prob,f} := \Exp_{\Sample \sim \Prob} f(\Sample),
\]
being understood that we implicitly assume that $f$ is integrable with respect to $\Prob$ when
using this notation.
This extends to finite signed measures: given a decomposition of $\HH \in \FSMSpace$ as $\HH = \alpha \Prob-\beta \Prob'$ with $\Prob,\Prob'$ two probability distributions and $\alpha,\beta \geq 0$, we denote
\[
\inner{\HH,f} := \alpha \inner{\Prob,f}-\beta \inner{\Prob',f},
\]
which can be checked to be independent of the particular choice of decomposition of $\HH$. With these notations, given a class $\GClass$ of measurable functions $g: \SampleSpace \to \RR\ \text{or}\ \CC$ we can define 
\[
\normfclass{\HH}{G} := \sup_{f \in \GClass} \abs{\inner{\HH,f}},
\]
and check that this is a semi-norm on the linear subspace $\set{\HH \in \FSMSpace: \forall f \in \GClass, f \text{ integrable w.r.t. } \HH}$ as claimed when we introduced~\eqref{eq:DefFNorm}. 
Similarly, pseudometrics similar to~\eqref{eq:DefYetAnotherMetric} can be extended to finite signed measures as
\[
d_{\GClass}(\HH,\HH') := \sup_{f \in \GClass} \abs{\abs{\inner{\HH,f}}^{2}-\abs{\inner{\HH',f}}^{2}}.
\]
When the functions in $\GClass$ are smooth these quantities can be extended to tempered distributions.

The total variation norm is defined on $\FSMSpace$ as $\normTV{\cdot}=\normfclass{\cdot}{B}$ with $\mathcal{B} = \set{f: f\ \text{is continuous and}\ \norm{f}_{\infty} \leq 1}$ \citep[see e.g. ][]{Sriperumbudur2010} 
and yields a Banach structure on $\FSMSpace$~\citep[see e.g.][]{Halmos_2013}. 

The mean kernel $\kappa$ (cf~\eqref{eq:DefMeanMapEmbedding}) can naturally be extended from probability distributions to finite signed measures. Let $\HH_1,\HH_2\in \FSMSpace$ and $\Prob_1,\Prob_1',\Prob_2,\Prob_2', \alpha_1,\alpha_2,\beta_1,\beta_2$ such that $\HH_1= \alpha_1\Prob_1 -\beta_1 \Prob_1'$ and  $\HH_2= \alpha_2\Prob_2 -\beta_2 \Prob_2'$ (decompositions as differences of probability measures). Provided that $\kernel(\cdot,\cdot)$ is well-defined on the corresponding probability distributions, we can define 
\begin{equation}\label{eq:DefKernelMeas}
 \kernel(\mu_1,\mu_2) := \alpha_1 \alpha_2 \kernel(\pi_1,\pi_2) - \alpha_1 \beta_2 \kernel(\pi_1,\pi_2') -\beta_1 \alpha_2 \kernel(\pi_1',\pi_2) + \beta_1 \beta_2 \kernel(\pi_1',\pi_2'),
\end{equation}
which can be checked to be independent of the particular choices of decomposition.

By linearity of the integral and the definition of the kernel for probability distributions, we obtain a \emph{seminorm} $\normkern{\cdot}$ associated to the mean kernel:
\begin{equation}\label{eq:DefNormKern}
\normkern{\HH}^2:=\iint\kappa(\sample,\sample')d\HH(\sample)d\HH(\sample')=\kappa(\HH,\HH),
\end{equation}
that coincides with the metric of the mean kernel~\eqref{eq:DefMMD} for  probability distributions.

\section{Proof of  Theorem~\ref{thm:LRIPsuff_excess}}
\label{sec:ProofLRIPsuff_excess}

In this section, we start with a suitable generalization of \citep[Section IV-A]{Bourrier2014},
working with some relaxed assumptions on the considered metrics.

\begin{definition}[hemimetric]\label{defhemimetric}
A function $\divg(\cdot\|\cdot): \mathcal{X} \times \mathcal{X} \to \mathbb{R}$ is a {\em hemimetric} if
\begin{align*}
\divg(x\|x) &= 0,\quad \forall x \in \mathcal{X}\\
\divg(x\|y) &\geq 0,\quad \forall x,y \in \mathcal{X}\\
\divg(x\|y) & \leq  \divg(x\|z)+\divg(z\|y),\quad \forall x,y,z \in \mathcal{X}
\end{align*}
A hemimetric is a {\em pseudometric} if it is symmetric: $\divg(x\|y) = \divg(y\|x)$ for any $x,y \in \mathcal{X}$. 

Hemimetrics on basic sets such as $\mathcal{X} = \mathbb{R}^{n}$ will be denoted $\divg(x\|y)$. 
Hemimetrics between probability distributions will be denoted $\divp(\pi\|\pi')$.
The notation $\distp(x,y)$ will preferentially be used for pseudometrics on basic sets, while $D(\pi,\pi')$ will denote pseudometrics between probability distributions. 

\end{definition}

\begin{definition}[relaxed lower restricted isometry property (rLRIP)]
A function $\Psi: \mathcal{X} \to \mathcal{Y}$ satisfies the lower restricted isometry property on the subset $\Sigma \subset \mathcal{X}$ 
with respect to the hemimetric $\dX(\cdot\|\cdot)$ on $\mathcal{X}$ and the pseudometric $\dY(\cdot,\cdot)$ on $\mathcal{Y}$ with constant $\eta \geq 0$ iff
\begin{equation}\label{eq:rLRIP}
\dX(x\|x') \leq   \dY(\Psi(x),\Psi(x'))+\eta,\quad \forall x,x' \in \Sigma.
\end{equation}
\end{definition}
\begin{lemma}\label{le:IOP}
Assume that $\Psi$ satisfies the rLRIP on $\Sigma$ with respect to $\dX(\cdot\|\cdot)$ and $\dY(\cdot,\cdot)$ with constant $\eta$.
Consider $\varepsilon,\nu \geq 0$ and a decoder $\Delta: \mathcal{Y} \to \Sigma \subset \mathcal{X}$ such that%
\begin{equation}
\label{decoder} \dY(y,\Psi(\Delta(y))) \leq (1+\nu) \inf_{z \in \Sigma} \dY(y,\Psi(z))+\varepsilon,\quad \forall y \in \mathcal{Y}.
\end{equation}
Then $\Delta$ satisfies the instance optimality property: $\forall x^{*} \in \mathcal{X},y \in \mathcal{Y}$,
\begin{equation}\label{eq:IOPgen}
  \dX(x^
  {*}\|\Delta(y)) 
\leq \distIOPgen(x^* \| \Sigma)
+(2+\nu) \dY(y,\Psi(x^{*}))
+\eta+\varepsilon,
\end{equation}
where
\begin{equation}\label{eq:distIOPgen}
\distIOPgen(x \| \Sigma) = 
\inf_{z \in \Sigma} \distIOPgen(x\|z); \qquad
\distIOPgen(x\|z):=
\dX(x^{*}\|z) + (2+\nu) \dY(\Psi(x^{*}),\Psi(z)).
  \end{equation}
\end{lemma}
\begin{proof}The proof follows very closely \citep{Bourrier2014} and is adaptated to the fact that $\dX(\cdot\|\cdot)$ is a hemimetric.
Consider $x^{*} \in \mathcal{X}$, $y \in \mathcal{Y}$ %
and $\hat{x} = \Delta(y)$. Consider any $z \in \Sigma$ and write
\begin{eqnarray*}
\dX(x^{*}\|\hat{x}) 
&\leq& \dX(x^{*}\|z)+\dX(z\|\hat{x}) \\
&\stackrel{rLRIP}{\leq} & \dX(x^{*}\|z)+ \dY(\Psi(z),\Psi(\hat{x}))+\eta\\
&\leq& \dX(x^{*}\|z)+ \dY(\Psi(z),y)+ \dY(y,\Psi(\hat{x}))	+\eta\\
&\stackrel{\eqref{decoder}}{\leq}& \dX(x^{*}\|z)+(2+\nu) \dY(y,\Psi(z))+\eta+\varepsilon\\
& \leq& \dX(x^{*}\|z) +(2+\nu) \dY(y,\Psi(x^{*}))+ (2+\nu) \dY(\Psi(x^{*}),\Psi(z))+\eta+\varepsilon\\
& =& \dX(x^{*}\|z) + (2+\nu) \dY(\Psi(x^{*}),\Psi(z))+(2+\nu) \dY(y,\Psi(x^{*}))+\eta+\varepsilon
\end{eqnarray*}
As this holds for any $z \in \Sigma$, taking the infimum yields the result.
\end{proof}
\begin{remark} Conversely, when $\dX(\cdot\|\cdot)$ is a {\em pseudometric}, if {\em some} decoder satisfies~\eqref{eq:IOPgen} for each $x^{*} \in \mathcal{X},y \in \mathcal{Y}$ with  $\distIOPgen(\cdot\|\cdot)$ {\em some} function such that $\distIOPgen(x\|\Sigma) = 0$ for each $x \in \Sigma$, then the rLRIP~\eqref{eq:rLRIP} holds with constant $2(\eta+\varepsilon)$ with respect to $\wt{\dY}(\cdot,\cdot) = (2+\nu) \dY(\cdot,\cdot)$. 
Indeed, for $x,x' \in \Sigma$, as $\distIOPgen(x\|\Sigma) = \distIOPgen(x'\|\Sigma) = 0$, by~\eqref{eq:IOPgen} and the symmetry of $\dX(\cdot\|\cdot)$, we have with $y:=\Psi(x)$, $\hat{x} := \Delta(y)$:
\begin{eqnarray*}
\dX(x\|\hat{x}) & \leq & \distIOPgen(x\|\Sigma)+ (2+\nu)\dY(y,\Psi(x)) + \eta+\varepsilon = \eta+\varepsilon\\
\dX(\hat{x}\|x') = \dX(x'\|\hat{x}) & \leq & \distIOPgen(x'\|\Sigma)+ (2+\nu)\dY(y,\Psi(x')) + \eta+\varepsilon = \wt{\dY}(\Psi(x),\Psi(x'))+ \eta+\varepsilon.
\end{eqnarray*}
The triangle inequality yields $\dX(x\|x') \leq \dX(x\|\hat{x})+\dX(\hat{x}\|x') \leq \wt{\dY}(\Psi(x),\Psi(x')) + 2(\eta+\varepsilon)$.
\end{remark}
To prove Theorem~\ref{thm:LRIPsuff_excess}, given a fixed $\hyp_0 \in \HypClass$, we apply Lemma~\ref{le:IOP} with 
\[
\dX(\Prob\|\Prob'):=\divp_{\hyp_0}(\Prob\|\Prob'),\quad \Psi(\cdot) = \SketchingOperatorProb(\cdot),\quad 
  \text{and}\quad \dY(\cdot,\cdot):=C_{\SketchingOperatorProb}\norm{\cdot-\cdot}_{2}.
  \]
   By~\eqref{eq:ineqdivg}-\eqref{eq:lowerRIP_excess} the rLRIP~\eqref{eq:rLRIP} is satisfied by $\Psi$ on $\Sigma := \Model$. By~\eqref{eq:ThmDecoder2} $y := \vy = \SketchingOperatorProb(\empProb)$ and $\Delta(\vy) = \estProb$ satisfy~\eqref{decoder} with $\varepsilon$ replaced by $C_{\SketchingOperatorProb} \varepsilon$. Since the definition~\eqref{eq:distIOPgen} with %
  yields $\distIOPgen(\Prob\|\Sigma) = \distIOPexgen_{\hyp_0}(\Prob\|\Model)$ as in~\eqref{eq:DefMDist2}, we get by~\eqref{eq:IOPgen} with $x^{*} := \Prob$:
  \[
    \divp_{\hyp_0}(\Prob\|\estProb) \leq \distIOPexgen_{\hyp_0}(\Prob,\Model) + (2+\nu) C_{\SketchingOperatorProb} \norm{\SketchingOperatorProb(\empProb)-\SketchingOperatorProb(\Prob)}+\eta+ C_{\SketchingOperatorProb} \varepsilon.
\]
On the other hand, using~\eqref{eq:ThmLearn2}, we obtain
\[
  \drisk_{\hyp_0}(\estProb,\hat{\hyp}) = \Risk(\estProb,\hat{\hyp}) - \Risk(\estProb,\hyp_0)
  \leq \varepsilon',
\]
and thus, by definition of $\divp_{\hyp_0}(\Prob\|\estProb)$ (cf~\eqref{lossdiv}):
\begin{align}
  \drisk_{\hyp_0}(\Prob,\wt{\hyp}) 
    &
      \leq \drisk_{\hyp_0}(\Prob,\hat{\hyp})-\drisk_{\hyp_0}(\estProb,\hat{\hyp})+\varepsilon'
   \leq \divp_{\hyp_0}(\Prob\|\estProb)+\varepsilon'. \label{eq:main-erm-ineq2} 
\end{align}
\section{Proof of Lemma~\ref{le:PointwiseConcentrationLemma} and 
Theorem~\ref{thm:mainLRIP}}
\label{sec:ProofThmMainLRIP}
To establish Lemma~\ref{le:PointwiseConcentrationLemma} we use Bernstein's inequality
for bounded random variables, which is for example a consequence of \citealt[Corollary 2.10]{Massart:2007book}.
\begin{lemma}[Bernstein's inequality]\label{lem:bernstein}
Let $X_i \in \mathbb{R}$, $i=1,\ldots,N$ be $i.i.d.$ bounded random variables such that $\mathbb{E}X_i=0$, $|X_i| \leq M$ and $Var(X_i) \leq \sigma^2$ for all $i$'s. Then for all $t>0$ we have
\begin{equation}
P\left(\frac{1}{N}\sum_{i=1}^N X_i \geq t\right)\leq \exp\left(-\frac{Nt^2}{2\sigma^2+2Mt/3}\right).
\end{equation}
\end{lemma}
\begin{proof}[Proof of Lemma~\ref{le:PointwiseConcentrationLemma}]
First, observe that for any $\FClass$-integrable probability distributions $\Prob,\Prob'$
\begin{align*}
\normkern{\Prob-\Prob'}^{2} 
&= \Exp_{\freq \sim \freqdist} 
\abs{\inner{\Prob,\rfeat}-\inner{\Prob',\rfeat}}^{2}
\leq \sup_{\freq \sim \freqdist} 
\abs{\inner{\Prob,\rfeat}-\inner{\Prob',\rfeat}}^{2} 
= \normfclass{\Prob-\Prob'}{\FClass}^{2}
\end{align*}
and that
\[
\frac{\norm{\SketchingOperatorProb(\Prob)-\SketchingOperatorProb(\Prob')}_{2}^{2}}{\normkern{\Prob-\Prob'}^{2}}-1
=
\frac{1}{m} \sum_{j=1}^{m} Z(\omega_{j})
\qquad \text{with} \qquad
Z(\omega)
:= 
\frac{
\abs{ \inner{\Prob,\rfeat}-\inner{\Prob',\rfeat}}^{2}
}{
\normkern{\Prob-\Prob'}^{2}
}-1
\]
Specializing to $\mProb,\mProb' \in \Model$ we get $1 \leq C := \normfclass{\mProb-\mProb'}{\FClass}/\normkern{\mProb-\mProb'} \leq \normfclass{\secant_{\kernel}}{F}$ and $-1\leq Z(\omega) \leq C^{2}-1$, hence $\normfclass{\secant_{\kernel}}{F} \geq 1$ and $|Z(\omega)|\leq \max\left(1,C^{2} -1\right) \leq C^{2} \leq \normfclass{\secant_{\kernel}}{F}^{2}$. As $\Exp_{\omega \sim \Lambda} Z(\omega) = 0$ we obtain
\begin{align*}
\mathtt{Var}_{\omega \sim \Lambda}(Z(\omega))
=& \mathtt{Var}_{\omega \sim \Lambda}\left(\tfrac{\abs{ \inner{\mProb-\mProb',\rfeat}}^{2}}{\normkern{\mProb-\mProb'}^{2}}\right)
\leq \frac{\Exp_{\omega \sim \Lambda}\abs{ \inner{\mProb-\mProb',\rfeat}}^{4}}{\normkern{\mProb-\mProb'}^{4}} \notag \\
\leq& \frac{\Exp_{\omega \sim \Lambda} \normfclass{\mProb-\mProb'}{\FClass}^{2} \cdot \abs{\inner{\mProb-\mProb',\rfeat}}^{2}}{\normkern{\mProb-\mProb'}^{4}} 
= \frac{\normfclass{\mProb-\mProb'}{\FClass}^2}{\normkern{\mProb-\mProb'}^2} = C^{2}.
\end{align*}
Applying Lemma~\ref{lem:bernstein} with the independent random variables $Z(\omega)$ we obtain for each $t>0$:
\[
\mathbb{P}
\left(
\abs{\frac{\norm{\SketchingOperatorProb(\mProb-\mProb')}_{2}^{2}}{\normkern{\mProb-\mProb'}^{2}}-1} \geq t
\right)
\leq 
2\exp\left(-\frac{m t^{2}}{2 C^{2} \cdot (1+t/3)}\right). \qedhere
\]
\end{proof}
\begin{lemma}\label{le:MainConcentrationLemmaAbstract}
Consider a family of functions $\FClass := \{\rfeat\}_{\freq \in \Omega}$, $\nMeasures$ parameters $(\freq_j)_{j=1}^{m}$ drawn i.i.d. according to some distribution $\Lambda$ on $\Omega$, and $\SketchingOperatorProb$ the (random) operator induced (see~\eqref{eq:SketchingOperatorProbDef}) by the feature function
\(
\SketchingOperator(\sample) := \tfrac{1}{\sqrt{m}} \left(\rfeatj(\sample)\right)_{j=1}^{m}.
\)
and $\kernel$ the associated kernel. 
Assume that~\eqref{eq:PointwiseConcentrationFn} holds with
concentration function $\ConcFn(t)$ 
and consider  $\secant \subset \secant_\kernel$ and $d_{\FClass}$ the metric defined in \eqref{eq:DefYetAnotherMetric}.
For any $\delta>0$ such that
\begin{equation}\label{eq:AssumeCoveringNumber}
N := \covnum{d_\FClass}{\secant}{\delta/2} < \infty,
\end{equation}
we have, with probability at least $1-2N \exp(-m / \ConcFn(\delta/2))$:
\begin{equation}\label{eq:MainRIP}
\sup_{\mu \in \secant} \abs{\norm{\SketchingOperatorProb(\mu)}_{2}^{2}-1} \leq \delta.
\end{equation}
\end{lemma}
\begin{proof}[Proof of Lemma~\ref{le:MainConcentrationLemmaAbstract}]
Consider $\mu = (\mProb-\mProb')/\normkern{\mProb-\mProb'}$ with $\mProb,\mProb' \in \Model$. By definition of the concentration function, for any $t>0$ and $\nMeasures \geq 1$
\begin{equation}\label{eq:PointwiseConcentrationAbstract}
\mathbb{P}
\left(
\abs{\norm{\SketchingOperatorProb(\mu)}_{2}^{2}-1} \geq t
\right)
\leq 
2\exp\left(-m / \ConcFn(t)\right).
\end{equation}
This establishes a pointwise concentration result when $\mu$ is on the normalized secant set $\secant_\kernel$. We now use a standard argument to extend this to a uniform result on $\secant$.
Let $\mu_{i}$, $1\leq i \leq N$ be the centers of a $\delta/2$-covering (with respect to the metric $d_{\FClass}$) of $\secant$. Using \eqref{eq:PointwiseConcentrationAbstract} with $t=\delta/2$, the probability that there is an index $i$ such that $\Big|\norm{\SketchingOperatorProb(\mu_i)}_{2}^{2}-1\Big| \geq \delta/2$ is at most $\probLevel = 2N \exp(-m / \ConcFn(\delta/2))$. Hence, with probability at least $1-\probLevel$, we have: for any $\mu \in \secant$, with $i$ an index chosen so that $d_{\FClass}(\mu,\mu_{i}) \leq \delta/2$:
\begin{align*}
\abs{\norm{\SketchingOperatorProb(\mu)}_{2}^{2}-1} &
\leq 
\abs{\norm{\SketchingOperatorProb(\mu)}_{2}^{2}-\norm{\SketchingOperatorProb(\mu_{i})}_{2}^{2}}
+
\abs{\norm{\SketchingOperatorProb(\mu_{i})}_{2}^{2}-1}\\
& \leq
\frac{1}{m}
\abs{
\sum_{j=1}^{m} 
\left(\abs{\inner{\mu,\rfeatj}}^{2}- \abs{\inner{\mu_{i}, \rfeatj}}^{2}\right)
}
+ \delta/2
 \leq
d_{\FClass}(\mu,\mu_{i}) + \delta/2 \leq \delta.\qedhere
\end{align*}
\end{proof}
\begin{proof}[Proof of Theorem~\ref{thm:mainLRIP}]
Denote $\probLevel = 2N \exp(-m / \ConcFn(\delta/2))$ with $N := \covnum{d_\FClass}{\secant_{\kernel}}{\delta/2}$. By Lemma~\ref{le:MainConcentrationLemmaAbstract}, the assumptions imply that with probability at least $1-\probLevel$ on the draw of $\omega_{j}$, $1 \leq j \leq m$, we have
\[
\sup_{\mu \in \secant_\kernel} \abs{\norm{\SketchingOperatorProb(\mu)}_{2}^{2}-1} \leq \delta.
\]
This implies~\eqref{eq:KernelRIPL2}. Since $\dnormloss{\secant_{\kernel}}{} < \infty$, the LRIP~\eqref{eq:lowerRIP_excess} holds wrt $\dnormloss{\cdot}{}$ with $C_\SketchingOperatorProb=\frac{\dnormloss{\secant_{\kernel}}{}}{\sqrt{1-\delta}}$. 
\end{proof}

\section{Proof of Lemma~\ref{lem:LemmaBiasTerm} and Lemma~\ref{lem:LemmaBiasTermBis}}\label{sec:proofbiasterm}

If $\Prob \in \ModelCT_\hyp$ then $0 = \Risk(\Prob,\hyp) = \Exp_{X \sim \Prob} \loss(X,\hyp) = \Exp_{X \sim \Prob} d^{p}(X,P_{\hyp}X)$ hence $d(X,P_{\hyp}X)=0$ almost surely, i.e., $X = P_{\hyp}X \in P_{\hyp} \SampleSpace = \mathcal{E}_{\hyp}$ almost surely. The converse is trivial. The bound~\eqref{eq:biasbound} follows directly since for any $h\in \HypClass$, $P_\hyp\Prob \in \ModelCT(\hyp) \subset \ModelCT(\HypClass)$.
This establishes the first claim of Lemma~\ref{lem:LemmaBiasTerm}.

Let $\hyp_0 \in \HypClass$ be fixed. By~\eqref{proj1}, with $Y \sim P_{\hyp_{0}}\Prob$, we have $P_{\hyp_{0}}Y = Y$ hence $\ell(Y,\hyp_{0}) = \divg^{p}(Y,P_{\hyp_{0}}Y) = 0$ and for
  any $\hyp \in \HypClass$:
\begin{align}
\drisk_{\hyp_{0}}(\Prob,\hyp)-\drisk_{\hyp_{0}}(P_{\hyp_{0}}\Prob,\hyp)
&=
\mathbb{E}_{X \sim \Prob} \ell(X,\hyp)-\mathbb{E}_{X \sim \Prob} \ell(X,\hyp_{0})
- \big(
    \underbrace{\mathbb{E}_{Y \sim P_{\hyp_{0}}\Prob} \ell(Y,\hyp)}_{
\mathbb{E}_{X \sim \Prob} \ell(P_{\hyp_{0}}X,\hyp)}-
    \underbrace{\mathbb{E}_{Y \sim P_{\hyp_{0}}\Prob} \ell(Y,\hyp_{0})}_0\big)
\notag \\
&=
\mathbb{E}_{X \sim \Prob} \left[\ell(X,\hyp)- \ell(X,\hyp_{0}) -\ell(P_{\hyp_{0}}X,\hyp)\right] \notag\\
&=
\mathbb{E}_{X \sim \Prob} \left[\divg^p(X,P_{\hyp}X)- \divg^p(X, P_{\hyp_{0}}X) -\divg^p(P_{\hyp_{0}}X , P_\hyp P_{\hyp_{0}}X) \label{eq:boundcomp1} \right].
\end{align}
For the second claim of Lemma~\ref{lem:LemmaBiasTerm}, by~\eqref{proj2}, since $d^{p}$ is a metric we have for any $x \in \SampleSpace$ 
\begin{align*}
\divg^{p}(x,P_{\hyp}x) 
& \leq  \divg^p(x, P_{\hyp}P_{\hyp_{0}}x) \leq  \divg^p(x,P_{\hyp_{0}}x)+\divg^p(P_{\hyp_{0}}x,P_{\hyp}P_{\hyp_{0}}x).
\end{align*}
It follows using~\eqref{eq:boundcomp1} that 
\[
\Delta \mathcal{R}_{\hyp_{0}}(\Prob,\hyp)-\Delta \mathcal{R}_{\hyp_{0}}(P_{\hyp_{0}}\Prob,\hyp) \leq 0.
\]
As this holds for any $\hyp$, and as equality is reached for $\hyp = \hyp_{0}$, we get $D_{\hyp_{0}}(\Prob\|P_{\hyp_0}\Prob) = 0$.\\
In particular when $p \in (0,1]$ we have $(a+b)^{p} \leq a^{p}+b^{p}$ for any $a,b \geq 0$ hence for $u,v,w \in \SampleSpace$, by the triangle inequality, $\divg^{p}(u,v) \leq  [\divg(u,w)+\divg(w,v)]^{p} \leq  \divg^p(u,w)+\divg^p(w,v)$, showing that $d^{p}$ is a metric.

For the claims of Lemma~\ref{lem:LemmaBiasTermBis}, we will exploit optimal transport through connections between the considered norms and the norm $\norm{\cdot}_{\LipClass(L,d)} = L\cdot \norm{\cdot}_{\LipClass(1,d)}$, where $\LipClass(L,d)$ denotes the class of functions $f: (\SampleSpace,d) \to \RR$ that are $L$-Lipschitz.

For $p\geq 1$, and $\SampleSpace$ with $d$-diameter bounded by $B$, since $\abs{a^{p}-b^{p}} \leq \max(p a^{p-1},pb^{p-1}) \abs{a-b}$ for any $a,b \geq 0$, we have
\[\abs{\loss(\sample,\hyp)-\loss(\sample',\hyp)} \leq pB^{p-1} \abs{d(\sample,P_\hyp\sample) - d(\sample',P_{\hyp}\sample')}\leq pB^{p-1}d(\sample,\sample'),\]
by the triangle inequality, hence $\LossClass(\HypClass) \subset \LipClass(pB^{p-1},d)$. 
Using \eqref{eq:ineqdivg} this 
implies that for any $\Prob,\Prob'$ we have in general in the above considered setting:
\[
\divp_{\hyp_0}(\Prob\|\Prob') \leq \dnormloss{\Prob - \Prob'}{} \leq 2 \normloss{\Prob-\Prob'}{} \leq 2p B^{p-1} \norm{\Prob-\Prob'}_{\LipClass(1,d)}.
\]
It is well-known that the 1-Wasserstein distance between two distributions can be equivalently
  defined in terms of optimal transport (so-called ``earth mover's distance'') but also as 
  \[
      \norm{\Prob-\Prob'}_{\textnormal{Wasserstein}_{1}(d)}
=  \norm{\Prob-\Prob'}_{\LipClass(1,d)}
   \]
as soon as $(\SampleSpace,d)$ is a separable metric space,  see, e.g., \citep[Theorem 11.8.2]{Dudley:2002ki}. 
By the transport characterization of the Wasserstein distance,  considering the transport plan that sends $\sample$ to $P_{\hyp} \sample$, where $\hyp \in \HypClass'$, we conclude 
\begin{equation}
  \label{eq:wastorisk1}
  \norm{\Prob-P_\hyp \Prob}_{\textnormal{Wasserstein}_{1}(d)}
   \leq \Exp_{\Sample \sim \Prob} d(\Sample,P_{\hyp}(\Sample))
   \leq \left[\Exp_{\Sample \sim \Prob} d^{p}(\Sample,P_{\hyp}(\Sample))\right]^{\tfrac{1}{p}} = \Risk(\Prob,\hyp)^{\frac{1}{p}},
 \end{equation}
by Jensen's inequality (since $p\geq 1$ here), yielding the claim~\eqref{eq:rboundgeneralcasepgeq1}.

For the final claim, we have
\[
  \norm{\SketchingOperatorProb(\Prob)-\SketchingOperatorProb(\Prob')}_2 
  = 
  \sup_{\norm{\vu}_{2} \leq 1} \abs{\inner{\SketchingOperatorProb(\Prob)-\SketchingOperatorProb(\Prob'),\vu}}
  = 
  \sup_{\norm{\vu}_{2} \leq 1} 
  \abs{E_{\Sample \sim \Prob} f_{\vu}(\Sample) -
    E_{\Sample \sim \Prob'} f_{\vu}(\Sample)}
\]
where $f_{\vu}(\sample) := \inner{\SketchingOperator(\sample),\vu}$. Moreover, for $\norm{\vu}_{2} \leq 1$ and any $\sample,\sample'$, since $\SketchingOperator$ is assumed $L$-Lipschitz:
\begin{align*}
  \abs{f_{\vu}(\sample)-f_{\vu}(\sample')}^2 =\inner{\SketchingOperator(\sample)-\SketchingOperator(\sample'),\vu}^2
  & \leq \norm{\SketchingOperator(\sample)-\SketchingOperator(\sample')}_{2}^2 \leq L^{2} d^{2}(\sample,\sample'),
\end{align*}
i.e., $f_{\vu}(\cdot)$ is $L$-Lipschitz with respect to $d(\cdot,\cdot)$.
It follows that for any $\Prob,\Prob'$, $\norm{\SketchingOperatorProb(\Prob)-\SketchingOperatorProb(\Prob')}_{2} \leq L \norm{\Prob-\Prob'}_{\LipClass(1,d)} = L\norm{\Prob-\Prob'}_{\textnormal{Wasserstein}_{1}(d)}$.  

The claim~\eqref{eq:skopboundpgeq1} when $p \geq 1$ follows by~\eqref{eq:wastorisk1}. 
When $p\leq 1$ and the space $\SampleSpace$ has $d$-diameter bounded by $B$, as $d(\Sample,P_{\hyp}\Sample) = 
d^{1-p}(\Sample,P_{\hyp}\Sample)d^{p}(\Sample,P_{\hyp}\Sample) \leq B^{1-p} d^{p}(\Sample,P_{\hyp}\Sample)$, we obtain~\eqref{eq:skopboundpleq1} as follows
\[
  \norm{\Prob-P_\hyp \Prob}_{\textnormal{Wasserstein}_{1}(d)}
\leq \Exp_{\Sample \sim \Prob} d(\Sample,P_{\hyp}(\Sample))
\leq B^{1-p} \Exp_{\Sample \sim \Prob} d(\Sample,P_{\hyp}(\Sample))^p=B^{1-p}\Risk(\Prob,\hyp).
\]

\section{Proof of Theorem~\ref{th:MainTheoremPCA} on Compressive PCA}
\label{sec:PCA_theory}
For Compressive PCA, we recall that $\PCAdim$ is the number of PCA components we want to estimate.
The hypothesis class $\HypClass$ is the set of linear subspaces of dimension $\PCAdim$ of the input space $\RR^\sampleDim$, which is in one-to-one correspondance with the space $\projspace_\PCAdim$ of orthoprojectors $\projP$ of rank $\PCAdim$.
In the remainder of this section we therefore use directly $\projspace_\PCAdim$ as the hypothesis class, for
notational convenience. We recall that for $r\geq \PCAdim$, we consider the model $\Model_r$ consisting of probability distributions having their second moment matrix of rank at most $r$.

Observe that for any $\projP \in \projspace_\PCAdim$:
\[
  \kPCArisk(\Prob,\projP):=\Exp_{\Sample \sim \Prob} \norm{\Sample-\projP \Sample}_{2}^{2}
  = \inner{\Cov_{\Prob},\mI-\projP}_F,\]
where the inner product $\inner{,}_F$ is the Frobenius product, and the minimum risk is 
\begin{equation}
\label{eq:PCAMinimumRisk}
\kPCArisk(\Prob,\optproj) = \inf_{\textrm{rank}(\mathbf{M}) \leq \PCAdim, \mathbf{M} \succcurlyeq 0} \norm{\Cov_\Prob-\mathbf{M}}_{\star} = \sum_{i>k} \lambda_i(\Cov_\Prob), %
\end{equation}
where $\optprojl$ ($1 \leq \ell \leq \sampleDim$) denotes an orthoprojector
onto the $\ell$ first eigenvectors of $\Cov_{\Prob}$, and $\eigv_i(\mathbf{M})$ denote the eigenvalues, with multiplicity and ordered in nonincreasing sequence, of a matrix $\mathbf{M}$.

We follow the improved risk analysis of Section~\ref{se:improvedanalysis}. In the above
  setting, the excess risk divergence~\eqref{lossdiv} with respect to an \emph{arbitrary} reference hypothesis $\projP_0 \in \projspace_\PCAdim$
  is given by
  \begin{equation}\label{eq:PCADiv}
    \divp_{\projP_0}(\Prob\|\Prob')  = \sup_{\mathbf{Q} \in \projspace_\PCAdim} \inner{\Cov_\Prob - \Cov_{\Prob'}, \projP_0-\mathbf{Q}}_F = \inner{\Cov_\Prob - \Cov_{\Prob'}, \projP_0}_F - \inf_{\mathbf{Q} \in \projspace_\PCAdim}
  \inner{\Cov_\Prob - \Cov_{\Prob'}, \mathbf{Q}}_F. 
    \end{equation}

\paragraph{Lower RIP.}
  We start with the following bound on the excess risk divergence, holding without restriction for any
  distributions $\Prob,\Prob'$ with existing second moments and any $\projP_0 \in \projspace_\PCAdim$:
  \begin{equation}
    \label{eq:bdivfrob}
  \divp_{\projP_{0}}(\Prob\|\Prob') \leq \sqrt{2\min(k,d-k)} \norm{\Cov_{\Prob} - \Cov_{\Prob'}}_F.
\end{equation}
\begin{proof}[Proof of~\eqref{eq:bdivfrob}]
By the so-called Ky Fan Theorem \cite{Fan:1949kz}, for a symmetric matrix $\mathbf{M} \in \RR^{\sampleDim \times \sampleDim}$, and
a positive integer $\ell \leq \sampleDim$, one has
\[
\sup_{\projP \in \projspace_\ell} \trace(\mathbf{M} \projP) = \sum_{i=1}^\ell \eigv_i(\mathbf{M})\,.
\]
As a result, we obtain
\begin{align*}
\divp_{\projP_{0}}(\Prob\|\Prob')
&=
   \sup_{\mathbf{Q} \in \projspace_\PCAdim} \inner{\Cov_{\Prob'}-\Cov_{\Prob},\mathbf{Q}}_F 
   + \inner{ \Cov_{\Prob}-\Cov_{\Prob'},\projP_{0}}_F\\
&\leq
\sum_{i=1}^{k} \lambda_{i}(\Cov_{\Prob'}-\Cov_{\Prob})+ \sum_{i=1}^{k} \lambda_{i}(\Cov_{\Prob}-\Cov_{\Prob'})
=
\sum_{i=1}^{k} \lambda_{i}(\Cov_{\Prob'}-\Cov_{\Prob}) - \sum_{i=d-k+1}^{d} \lambda_{i}(\Cov_{\Prob'}-\Cov_{\Prob})\\
&=
\sum_{i=1}^{\min(k,d-k)} \lambda_{i}(\Cov_{\Prob'}-\Cov_{\Prob}) - \sum_{i=d-\min(k,d-k)+1}^{d} \lambda_{i}(\Cov_{\Prob'}-\Cov_{\Prob})\\
  & \leq \sqrt{2\min(k,d-k)}\sqrt{\sum_{i=1}^d\lambda_{i}^{2}(\Cov_{\Prob'}-\Cov_{\Prob})} = \sqrt{2\min(k,d-k)} \norm{\Cov_{\Prob'}-\Cov_{\Prob}}_F.
\end{align*}
\end{proof}
Since the right-hand side of~\eqref{eq:bdivfrob} does not depend of $\projP_0$, we get from~\eqref{eq:ineqdivg} in particular that  for any $\Prob,\Prob'$ with finite second moments
\[
  \dnormloss{\Prob - \Prob'}{} \leq \sqrt{2k} \norm{\Cov_{\Prob'}-\Cov_{\Prob}}_F.
\]
Hence, since $\mathcal{M}$ is a linear operator having a RIP~\eqref{eq:sym_RIP} on
matrices of rank lower than $2r$, $\mathcal{M}$ induces (in the way described in Section~\ref{sec:CompressivePCA}) a sketching operator $\SketchingOperatorProb: \Prob \mapsto \SketchingOperatorProb(\Prob) := \mathcal{M}(\Cov_{\Prob})$ that has the lower RIP described in~\eqref{eq:lowerRIP_excess} with constants $ C_\SketchingOperatorProb = \frac{\sqrt{2\PCAdim}}{\sqrt{1-\delta}},\eta=0$ on model $\Model_r$.

\paragraph{Ideal decoder and generic excess risk control.} The ideal decoder~\eqref{eq:DefIdealDecoder} writes
\[
\Decoder[\vy]
:=\argmin_{\mProb \in \Model_{r}} \norm{\SketchingOperatorProb(\mProb)-\vy}_{2}^{2},
\]
 and is equivalent to~\eqref{eq:PCAIdealDecoder}, i.e.
  \[
    \hat{\Cov} :=\argmin_{\Cov: \textrm{rank}(\Cov) \leq r; \Cov \succcurlyeq 0} \norm{\mathcal{M}(\Cov)-\vy}_{2}^{2}.
  \]
  Formally, $\Decoder[\vy]$ can then be taken as any distribution having second
  moment matrix $\hat{\Cov}$. This last step can naturally be shunted since whatever the choice
  of such a representative distribution, the associated estimated hypothesis $\hat{\hyp}$ is directly given by~\eqref{eq:PCAIdealDecoderBestHyp}.

By Theorem~\ref{thm:LRIPsuff_excess}, this decoder is instance optimal yielding~\eqref{eq:MainBoundExcessRisk} with $\nu = \eta=\varepsilon=\varepsilon'=0$,
i.e., the excess risk of $\hat{\hyp}$  of the procedure of Section~\ref{sec:CompressivePCA}
when the true data distribution is $\Prob$ is controlled by 
\begin{equation}
  \label{eq:PCAbound1}
   \distIOPexgen_{\hyp^{\star}_{\Prob}}(\Prob,\Model_r) +
   2C_\SketchingOperatorProb \norm{\SketchingOperatorProb (\Prob)-\SketchingOperatorProb(\empProb)}_2
   =
   \distIOPexgen_{\hyp^{\star}_{\Prob}}(\Prob,\Model_r) +
   2C_\SketchingOperatorProb \norm{\mathcal{M} (\Cov_{\Prob}-\Cov_{\empProb})}_2,
   \end{equation}
with the bias term defined in~\eqref{eq:DefMDist2}.

\paragraph{Control of the bias term.}

The next lemma improves over~Lemma~\ref{lem:LemmaBiasTerm}  in the special case of PCA.
\begin{lemma}\label{lem:BiasPCA}
Consider a probability distribution $\Prob$ with finite second moments, $\projP^*:=\projP^{*[k]}_\Prob$, and $\mProb_r \in \Model_r$  (any) probability distribution with covariance $\Cov_\Prob^{[r]} := \projP^{*[r]} \Cov_\Prob \projP^{*[r]}$.
We have $\divp_{\projP^*}(\Prob\|\mProb_r) = 0$.
\end{lemma}
\begin{proof}
  Since $\Cov_\Prob - \Cov_\Prob^{[r]}$ is a nonnegative matrix and $r \geq k$, \eqref{eq:PCADiv} yields
\[
  \divp_{\projP^*}(\Prob\|\mProb_r) =
  \underbrace{\inner{\Cov_\Prob - \Cov^{[r]}_{\Prob}, \projP^* }_F}_{=0} -
  \inf_{\mathrm{Q} \in \projspace_k}\underbrace{
  \inner{\Cov_\Prob - \Cov^{[r]}_{\Prob}, \mathbf{Q}}_F}_{\geq 0} = 0. \qedhere
\]
\end{proof}
As a result the ``bias'' term in the bound~\eqref{eq:PCAbound1} is upper bounded as
\begin{equation}
  \label{eq:biasPCA}
   \distIOPexgen_{\hyp^{\star}_{\Prob}}(\Prob,\Model_r) =   \inf_{\mProb \in \Model_r} \paren{\divp_{\projP^*}(\Prob\|\mProb)
    + 2 C_\SketchingOperatorProb \norm{ \SketchingOperatorProb(\Prob)-\SketchingOperatorProb(\mProb)}_2} \leq 2
  C_\SketchingOperatorProb
  \norm{\mathcal{M}(\Cov_\Prob - \Cov_\Prob^{[r]})}_2.
\end{equation}

From now on we assume $d>k$ (otherwise the bias is trivially 0). We now use the following lemma to bound $\norm{\mathcal{M}(\Cov_\Prob - \Cov_\Prob^{[r]})}_2.$
  \begin{lemma}\label{le:Skeleton}
Let $\Cov \in \mathbb{R}^{d \times d}$ be symmetric p.s.d. and $\mathcal{M}: \mathbb{R}^{d \times d} \to \mathbb{R}^{m}$ satisfy the upper RIP in \eqref{eq:sym_RIP}, i.e. ${\norm{\mathcal{M}(\mathbf{M})}_{2}^{2}}/{\norm{\mathbf{M}}_{F}^{2}} \leq 1+\delta$ for all matrices of rank at most $2r$. For $1\leq \ell \leq d$, consider $\Cov^{[\ell]}$ a best rank $\ell$ approximation to $\Cov$. Then for any $1 \leq s \leq \min(\ell,2r)$ we have 
\[
\norm{\mathcal{M}(\Cov-\Cov^{[\ell]})}_{2} \leq \sqrt{1+\delta} \frac{\sigma_{\ell-s+1}}{\sqrt{s}},
\]
where $\sigma_{j} := \sigma_{j}(\Cov) := \sum_{i=j+1}^{d} \lambda_{i}(\Cov)$. 
\end{lemma}
\begin{proof}
  This proof follows mainly the ideas of~\cite{Candes:2008aa}.
  By definition of $\Cov^{[\ell]}$ there is an eigendecomposition $\Cov = U \Lambda U^{T}$, where $\Lambda$ is a diagonal matrix containing the eigenvalues of $\Cov$ with multiplicity in decreasing order, such that $\Cov^{[\ell]} = U \Lambda^{[\ell]}U^{T}$ where $\Lambda^{[\ell]}$ contains the first $\ell$ eigenvalues. Decompose $\Lambda$ into blocks, $\Lambda = \sum_{j \geq 0} \Lambda_{j}$, where $\Lambda_0$ contains the first $\ell-s$ eigenvalues, and $\Lambda_{j}$, $j \geq 1$ are the next blocks of $s$ eigenvalues
  in decreasing order (the last block is of size $\leq s$); that is, for $j\geq 1$ the block $\Lambda_{j}$ contains eigenvalues of indices
  $m$ such that $\ell + (j-2) s < m \leq \ell + (j-1)s$. Let us also denote $\Lambda_j^{+}$, for $j \geq 1$,
  the blocks of eigenvalues of size $s$ starting one index later,
  that is, $\Lambda_{j}$ contains eigenvalues of indices $m$ such that $\ell + (j-2) s +1 <  m \leq \ell + (j-1)s + 1 $.
  Let $S_{j} = U \Lambda_{j} U^{T}$, so that $\Cov-\Cov^{[\ell]} = \sum_{j \geq 2} S_{j}$. As $s \leq 2r$,  $\mathrm{rank}(S_{j}) \leq 2r$, $j \geq 2$ hence by the upper RIP property, 
\begin{align*}
\tfrac{1}{\sqrt{1+\delta}} \norm{\mathcal{M}(\Cov-\Cov^{[\ell]})}_{2}
& =  
\tfrac{1}{\sqrt{1+\delta}}\norm{\sum_{j \geq 2} \mathcal{M}(S_{j})}_{2}
\leq \tfrac{1}{\sqrt{1+\delta}}\sum_{j \geq 2} \norm{\mathcal{M}(S_{j})}_{2}
\leq \sum_{j \geq 2} \norm{S_{j}}_{F}\\
&= \sum_{j \geq 2} \norm{\Lambda_{j}}_{2}
\leq  \sum_{j \geq 2} \sqrt{s} \norm{\Lambda_{j}}_{\infty}
\leq  \sum_{j \geq 2} \sqrt{s} \frac{\norm{\Lambda_{j-1}^+}_{1}}{s}
= \sum_{j \geq 1} \frac{\|\Lambda_{j}^+\|_{1}}{\sqrt{s}} = \frac{\sigma_{\ell-s+1}}{\sqrt{s}}.
\end{align*}
\end{proof}
The above lemma (with $\ell:=r$)  allows us to control $\norm{\SketchingOperatorProb(\Prob)-\SketchingOperatorProb(\Prob_{r})} = \norm{\mathcal{M}(\Cov_{\Prob}-\Cov_{\Prob}^{[r]})}_{2}$ using $\sigma_{r-s+1}$ for $1 \leq s \leq r$.
This, in combination with~\eqref{eq:biasPCA} and~\eqref{eq:PCAbound1}, establishes~\eqref{eq:generalPCARiskBound}.
Moreover, choosing $s:=r-k+1$, we get $\sigma_{r-s+1} = \sigma_k=\kPCArisk(\Prob,\optproj)$,
leading to~\eqref{eq:MainPCARiskBound}.This proves Theorem~\ref{th:MainTheoremPCA}.

\twocolumn[\section*{Table of notations}]
\newcommand{\refone}[1]{\ref{#1}}
\newcommand{\eqrefone}[1]{\eqref{#1}}

\begin{supertabular}{|p{0.136\textwidth}p{0.313\textwidth}|}
\hline
$\sample \in \SampleSpace$ & sample and sample space\\
$\vy$ & sketch vector \eqrefone{eq:GenericSketching} \\
$\SketchingOperator$ & sketching function \eqrefone{eq:GenericSketching}, \eqrefone{eq:DefRandomFeatureFunction}, \\
$\SketchingOperatorProb$ & sketching operator \eqrefone{eq:SketchingOperatorProbDef} \\
\hline
$\Prob,\mProb$ & probabilities on sample space\\
$\HH$, $\nu$ & measures on sample space\\
$\inner{\Prob,f}$ &  $\Exp_{X \sim \Prob}f(X)$\\
$\inner{\HH,f}$ & $\int f(x) d\HH(x)$  (App. \refone{sec:FiniteSignedMeasures}) \\
\hline
$\hyp$ & hypothesis\\
$\HypClass$ & class of hypotheses \\
$\loss(\cdot,\hyp)$ & loss function  \\
$\Risk, \drisk_{\hyp}$ & risk \eqrefone{eq:BestHyp}, excess risk (Def. \refone{def:excessriskdiv}) \\
$\hyp^\star=\hyp^{\star}_{\Prob}$ & best hypothesis \eqrefone{eq:BestHyp} \\
$\proxyRisk$ & proxy for the risk \eqrefone{eq:DefRiskProxy}, \eqrefone{eq:GenericRiskProxy} \\
$\hat{\hyp}$ & learned hypothesis \eqrefone{eq:DefRiskProxy} \\
$P_{\hyp}$ & projection function for comp.-type task (Def. \refone{def:comptypetask}) \\
\hline
$\LossClass = \LossClass(\HypClass)$ & class of loss functions \eqrefone{eq:DefRiskNorm} \\
$\DLossClass = \DLossClass(\HypClass)$ & class of loss differences \eqrefone{eq:DefDLossClass} \\ 
\hline
$\FClass = \set{\rfeat}_{\freq \in \Omega}$ & class of features (Def. \refone{def:RFmap}) \\ 
$\freqdist$ & probability distribution of feature parameters $\freq$ (Def. \refone{def:RFmap}) \\
\hline
$\kernel(x,x')$ & psd kernel (Def. \refone{def:MMD}) \\
$\kernel(\Prob,\Prob')$ & kernel mean embedding \eqrefone{eq:DefMeanMapEmbedding} \\
$C_{\SketchingOperatorProb}$,$C_{\kernel}$, $C_{\kernel_{\SketchingOperator}}$ & kernel constants \eqrefone{eq:lowerRIP}; \eqrefone{eq:KernelLRIP}; \eqrefone{eq:KernelRIPL2} \\
\hline
$\normfclass{\HH}{G}$ & $\sup_{f \in \GClass} \abs{\inner{\HH,f}}$
(\eqrefone{eq:DefFNorm}, App. \refone{sec:FiniteSignedMeasures}) \\
$\normkern{\HH}$ & MMD norm \eqrefone{eq:DefMMD}, \eqrefone{eq:DefNormKern} \\
$\normloss{\cdot}{}$, $\dnormloss{\cdot}{}$ & task-driven norms \eqrefone{eq:DefRiskNorm}; \eqrefone{eq:ineqdivg} \\
\hline
$ \divp_{\hyp}(\Prob\|\Prob')$ & excess-risk divergence \eqrefone{lossdiv} \\
$\distIOPexgen_{\hyp}(\Prob,\Model)$, & bias term wrt. model \eqrefone{eq:DefMDist2} \\
$d_{\FClass}(\Prob,\Prob')$ & feature-based metric \eqrefone{eq:DefYetAnotherMetric} \\
\hline
$\Model$ & model set (of probabilities)\\
$\Model_{\hyp}$ & probabilities s.t. $\hyp$ optimal \eqrefone{eq:DefSimpleModelSetGivenHyp} \\
$\ModelCT_{\hyp}$, $\ModelCT(\HypClass)$ & compression-type model set \eqrefone{eq:DefLeastRestrictedModel} \\
$\ModelML_{\hyp}$, $\ModelML(\HypClass)$ & max. likelihood model set \eqrefone{eq:DefNaturalMLModel} \\
$\secant = \secant_{\kernel}(\Model)$ & normalized secant set \eqrefone{eq:DefNormalizedSecantSet}\\
\hline
$\norm{\mathcal{E}}$ & radius of a set of measures \eqrefone{eq:DefSetRadius} \\
$\ConcFn(t)$ & concentration function \eqrefone{eq:PointwiseConcentrationFn} \\
$\covnum{\norm{\cdot}}{A}{\eps}$ & covering numbers (Def. \refone{def:covnum}) \\
$\Ball$ & Closed ball (Def. \refone{def:covering}) \\ 
\hline
\end{supertabular}
 \onecolumn
 \pagebreak
\bibliographystyle{abbrvnat}
\bibliography{complearn_preprint}
\end{document}